\documentclass[sigconf, nonacm]{acmart}

\AtBeginDocument{%
  \providecommand\BibTeX{{%
    \normalfont B\kern-0.5em{\scshape i\kern-0.25em b}\kern-0.8em\TeX}}}

\copyrightyear{2021}
\acmYear{2021}
\setcopyright{acmcopyright}
\acmConference[KDD '21] {Proceedings of the 27th ACM SIGKDD Conference on Knowledge Discovery and Data Mining}{August 14--18, 2021}{Virtual Event, Singapore.}
\acmBooktitle{Proceedings of the 27th ACM SIGKDD Conference on Knowledge Discovery and Data Mining (KDD '21), August 14--18, 2021, Virtual Event, Singapore}
\acmPrice{15.00}
\acmISBN{978-1-4503-8332-5/21/08}
\acmDOI{10.1145/3447548.3467258}

\DeclareMathOperator{\PDAG}{PDAG}
\DeclareMathOperator{\DAG}{DAG}

\DeclareMathOperator{\Adj}{\mathbf{Adj}}

\DeclareMathOperator{\Dp}{\mathbf{Dp}}
\DeclareMathOperator{\Ds}{\mathbf{Ds}}
\DeclareMathOperator{\Ip}{\mathbf{Ip}}
\DeclareMathOperator{\Is}{\mathbf{Is}}

\newcommand{\g}[1][G]{\mathcal{#1}}

\newcommand{\pathsubset}[1][T]{\mathbf{#1}}

\newcommand{\methodname}{FACTS}

\newtheorem{theorem}{Theorem}[section]

\newtheorem{lemma}[theorem]{Lemma}

\newtheorem{proposition}[theorem]{Proposition}

\usepackage{algorithm}
\usepackage{algpseudocode}
\usepackage{subfigure}
\usepackage[]{algorithmicx}
\usepackage{multirow}
\usepackage{amsmath}

\newcommand{\nind}{\not\!\perp\!\!\!\perp}
\newcommand{\ind}{\!\perp\!\!\!\perp}
\algnewcommand{\LineComment}[1]{\State \(\triangleright\) #1}

\usepackage{transparent,colortbl,xcolor}
\definecolor{red}{RGB}{255,0,81}
\definecolor{blue}{RGB}{0,139,251}

\begin{document}
	\fancyhead{}

%%
%% The "title" command has an optional parameter,
%% allowing the author to define a "short title" to be used in page headers.
\title{Explaining Algorithmic Fairness Through Fairness-Aware Causal Path Decomposition}

%%
%% The "author" command and its associated commands are used to define
%% the authors and their affiliations.
%% Of note is the shared affiliation of the first two authors, and the
%% "authornote" and "authornotemark" commands
%% used to denote shared contribution to the research.

\author{Weishen Pan$^{1}$, Sen Cui$^{1}$, Jiang Bian$^{2}$, Changshui Zhang$^{1}$, Fei Wang$^{3}$}
\affiliation{%
  \institution{$^1$Institute for Artificial Intelligence, Tsinghua University (THUAI), State Key Lab of Intelligent Technologies and Systems, Beijing National Research Center for Information Science and Technology (BNRist), Department of Automation, Tsinghua University, P.R.China\\$^2$Department of Health Outcomes and Biomedical Informatics, College of Medicine, University of Florida, USA\\$^3$Department of Population Health Sciences, Weill Cornell Medicine, USA}
  \streetaddress{}
  \city{}
  \state{}
  \country{}
  \postcode{}
}
\email{{pws15,cuis19}@mails.tsinghua.edu.cn, bianjiang@ufl.edu}
\email{zcs@mail.tsinghua.edu.cn, few2001@med.cornell.edu}

\renewcommand{\authors}{Weishen Pan, Sen Cui, Jiang Bian, Changshui Zhang, Fei Wang}

%%
%% The abstract is a short summary of the work to be presented in the
%% article.
\begin{abstract}
Algorithmic fairness has aroused considerable interests in data mining and machine learning communities recently. So far the existing research has been mostly focusing on the development of quantitative metrics to measure algorithm disparities across different protected groups, and approaches for adjusting the algorithm output to reduce such disparities. In this paper, we propose to study the problem of identification of the source of model disparities. Unlike existing interpretation methods which typically learn feature importance, we consider the causal relationships among feature variables and propose a novel framework to decompose the disparity into the sum of contributions from fairness-aware causal paths, which are paths linking the sensitive attribute and the final predictions, on the graph. We also consider the scenario when the directions on certain edges within those paths cannot be determined. Our framework is also model agnostic and applicable to a variety of quantitative disparity measures. Empirical evaluations on both synthetic and real-world data sets are provided to show that our method can provide precise and comprehensive explanations to the model disparities.
\end{abstract}

%%
%% The code below is generated by the tool at http://dl.acm.org/ccs.cfm.
%% Please copy and paste the code instead of the example below.
%%


%%
%% Keywords. The author(s) should pick words that accurately describe
%% the work being presented. Separate the keywords with commas.

%% A "teaser" image appears between the author and affiliation
%% information and the body of the document, and typically spans the
%% page.
%%
%% This command processes the author and affiliation and title
%% information and builds the first part of the formatted document.
\maketitle

\section{Introduction}
Machine learning algorithms have been widely applied in a variety of real-world applications including high-stakes scenarios such as loan approvals, criminal justice, healthcare, etc. In these real-world applications, fairness is getting increasing attentions as machine learning algorithms may lead to discrimination against certain disadvantaged sub-populations. This triggers the research on algorithmic fairness, which focus on whether members of specific unprivileged groups are more likely to receive unfavorable decisions made from machine learning algorithms.

One important line of research in the computational fairness community is to develop metrics for measuring group fairness, such as demographic parity \cite{dwork2012fairness}, equalized opportunity \cite{hardt2016equality}, accuracy parity \cite{zafar2017fairness}, etc., so that the discrepancy among the decisions made in different groups (a.k.a. disparity) are precisely quantified, which can further inspire the development of fair machine learning models that aim to minimize such disparities \cite{zafar2017fairness,agarwal2018reductions,zhang2018mitigating,madras2018learning}.

Despite the great efforts on fairness quantification and fair model development, one critical issue that has not been studied extensively is the diagnostics of model fairness, i.e., {\em what are the reasons that lead to the model disparity?} This information is crucial for understanding the intrinsic model mechanism and provides insights on how to improve model fairness. As an example, under the current pandemic, researchers have found that the racial and ethnic minority groups have been disproportionately affected by COVID-19. African Americans and Hispanics or Latinos are found to be more likely to have positive tests \cite{adegunsoye2020association,martinez2020sars}, COVID-19 associated hospitalizations and deaths \cite{price2020hospitalization}, compared with non-Hispanic Whites. In this case, it is crucial to figure out whether such disparity is coming from genetic factors or accessibility to adequate healthcare services, which will imply completely different clinical management plans and public health policies for battling with the pandemic.

In view of this need, recently researchers have leveraged Shapley value based methods \cite{lundberg2017unified} to attribute the model disparity as the sum of individual contributions from input features \cite{lundberg2020fairness, begley2020explainability}, so that we can understand which feature contributes more or less to the model disparity. However, in real-world problems, the mechanism that causes model disparity could be much more complex. Considering the COVID-19 example above, it turns out that one main factor contributing to such disparity is disproportionate access to care for patients with different races and ethnicity, which can be impacted by both economic status \citep{gould2020black} and food insecurity \cite{wolfson2020food}. In practice, they correspond to two different causal paths leading to outcome disparity and imply different public health intervention policies.

In this paper, we propose \methodname~(which stands for Fairness-Aware Causal paTh decompoSition), a novel framework for algorithm fairness explanation. \methodname~decomposes the model disparity as the sum over the contributions of a set of Fairness-Aware Causal paThs (FACT) linking the sensitive attributes with the outcome variable. In this way, our approach can quantify the different causality mechanisms that can lead to the overall model disparity. Specifically, \methodname~includes two steps:
\begin{itemize}
	\item {\em Step 1. FACTs identification}, where we propose a method to identify all active paths that link the sensitive attributes and final outcome without colliders given a causal graph constructed on feature variables (which are referred to as FACTs). The graph could be given according to domain knowledge or learned from the training data. One important consideration here is that frequently the causal directions for certain edges on the graph cannot be determined \cite{borboudakis2012incorporating,perkovic2020identifying}, which makes the graph a Partially Directed Acyclic Graph (PDAG). Our proposed algorithm can effectively identify active paths on PDAGs.
	
	\item {\em Step 2. Disparity attribution through Shapley value decomposition}, where we propose a Shapley value \cite{lundberg2017unified} based method to attribute the quantified model disparity value (e.g., according to demographic parity  \cite{dwork2012fairness}) to identified FACTs, so that the contribution of each FACT can be quantified.
\end{itemize}

In addition, with the derived attributed disparity on FACTs, we further develop a fair learning approach by selectively removing the FACTs based on their effects on disparity and accuracy through data transformation. Our framework is model agnostic and can be applied to a broad set of popular group-fairness criteria including demographic parity, equalized opportunity, equalized odds, and accuracy parity. 

With experiments on both synthetic and real-world datasets, we show that \methodname~can accurately quantify individual path contributions to the model disparity. With qualitative analysis on real-world datasets, we demonstrate how our approach can appropriately explain the sources of disparity and successfully make fair adjustments\footnote{We upload our source code on https://github.com/weishenpan15/FACTS.}.

\section{Preliminaries and Related Works}
\subsection{Notations}
In this paper, we use capitalized/lower-case letters in italics to represent a variable/value of the variable. We use capitalized/lower-case letters in boldface to represent a variable set/values of the variable set. We use $\{\dots\}$ to represent a set without ordering relationship and use $[\dots]$ to represent a sequence. $\pi$ is a function indicating the rank of specific elements in an ordered sequence. For example, for an order pair of features $[X_1, X_2]$, we will have $\pi(1) < \pi(2)$. 

Suppose we are given a dataset including samples characterized by a set of variables $\{A, \mathbf{X}, Y\}$, where $\mathbf{X} = \{X_1,...,X_M\}$ is the set of input feature variables. $A \in \{0,1\}$ is the sensitive attribute and $Y \in \{0,1\}$ is the outcome. We set $Y=1$ as the favored outcome. $f$ is a trained model and $\hat{Y}$ is the predicted outcome. $\mathbf{x}=(x_1,x_2,\cdots,x_M)^\top$ is a concrete data vector with $X_1=x_1,X_2=x_2,\cdots,X_M=x_M$.

\subsection{Causal Model}
\label{sec:causal_models}
We introduce causal model-related concepts that will be used throughout the paper in this subsection. Our descriptions are based on the definitions in \cite{spirtes2000causation,perkovic2020identifying}.

\noindent \textbf{Nodes and Edges.} A graph $\g$ consists of a set of nodes (variables) $\{A, Y, \hat{Y}, X_{1},\dots,X_{M}\}$ and a set of edges (variable relations). The edges can be either directed ($\rightarrow$) and undirected ($-$). We call two nodes are adjacent to each other if there is an edge linking them. The collection of nodes adjacent to $X_i$ is denoted as $\Adj(X_i)$.

\noindent \textbf{Paths.} A \textit{path} $p$ from $X_i$ to $X_j$ in $\g$ is a sequence of nodes where every successive nodes are adjacent in $\g$. A path from $X_i$ to $X_j$ in which all edges are directed from $X_i$ towards $X_j$ ($X_i \rightarrow \dots \rightarrow X_j$) is a \textit{causal path} from $X_i$ to $X_j$.

\noindent \textbf{Causal Relationships.} $X_i$ is a \textit{parent} of $X_j$ if there is a directed edge from $X_i$ to $X_j$, and $X_j$ is a \textit{child} of $X_i$. $X_i$ is an \textit{ancestor} of $X_j$ if there is a causal path from $X_i$ to $X_j$, and $X_j$ is a \textit{descendant} of $X_i$. Following prior research \cite{kilbertus2017avoiding,baer2019fairness}, $\hat{Y}$ is the child of all $X_i$, which means the predictor $f$ maps the features variables $\mathbf{X}$ to predicted output $\hat{Y}$. 

\noindent \textbf{DAGs and PDAGs.}
A \textit{directed graph} is a graph where all edges are directed. A directed graph without directed cycles is a \textit{directed acyclic graph $(\DAG)$}, where \textit{directed cycle} is formed as a causal path from $X_i$ to $X_j$ plus a directed edge $X_j \to X_i$. A \textit{partially directed graph} is a graph where edges can be either directed or undirected. Similarly, a \textit{partially directed acyclic graph $(\PDAG)$} is a partially directed graph without directed cycles. 

\begin{figure}[tbp]
	\centering
	\includegraphics[width=1.5in]{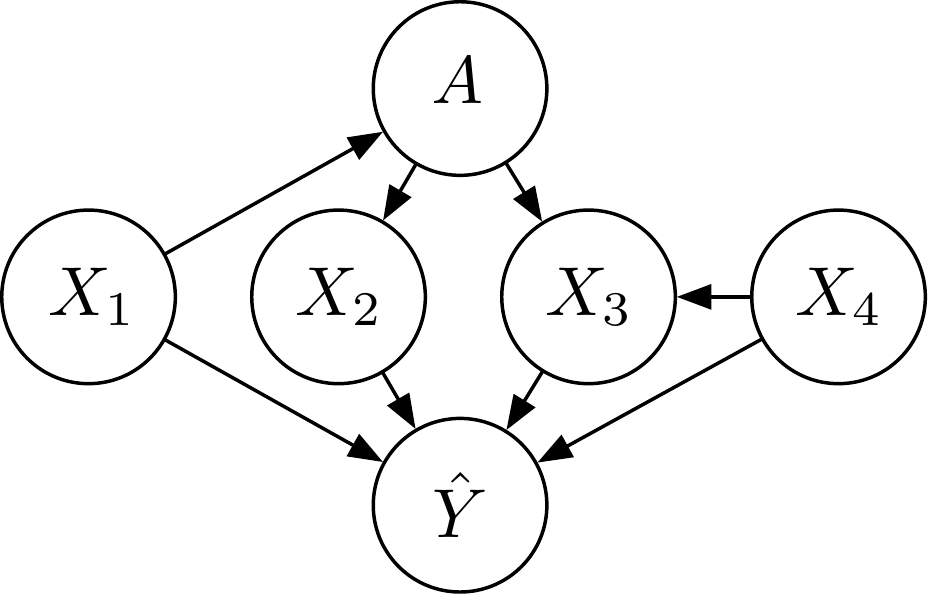}
	\caption{An example of a causal graph. Here the prediction $\hat{Y}$ is obtained by a function $f$ which takes $X_1, \dots ,X_4$ as input features.}
	\label{fig:intro_examples}
	\vspace{-2.0em}
\end{figure}

\noindent \textbf{Colliders, Active Nodes and Paths.} Even though there are paths linking $X_i$ and $X_j$ in $\g$, $X_i$ and $X_j$ are not guaranteed to be dependent. In the example shown in Figure \ref{fig:intro_examples}, there are paths (e.g., $A \rightarrow X_3 \leftarrow X_4$) linking $A$ and $X_4$, but $A$ and $X_4$ are independent due to fact that they are linked by a collider $X_3$. To characterize the variable relationships, we introduce the following definitions.

If a path $p$ contains $X_i \rightarrow X_k \leftarrow X_j$ as a subpath, then $X_k$ is a \textit{collider} on $p$. For a given set of nodes $\mathbf{C}$ (referred to as the {\em conditioning set}), a node $X_i$ is active relative to $\mathbf{C}$ on a path $p$ if either: 1) $X_i$ is a not a collider on $p$ and not in $\mathbf{C}$; 2) $X_i$ is a collider, and $X_i$ or any of its descendants is in $\mathbf{C}$. A path is active relative to $\mathbf{C}$ only when every node on the path is active relative to $\mathbf{C}$. When $\mathbf{C} = \emptyset$, the definition of an active path is degenerated to that there is no collider on the path. When we say $p$ is an active path, we means $p$ is an active path relative to $\emptyset$ in this paper.

As in Figure \ref{fig:intro_examples}, considering the path $A \rightarrow X_3 \rightarrow \hat{Y}$ when $\mathbf{C} = \emptyset$, $X_3$ is an active node since $X_3$ is not a collider. Thus $A \rightarrow X_3 \rightarrow \hat{Y}$ is an active path. Similarly, $A \rightarrow X_3 \leftarrow X_4$ is not an active path because $X_3$ is a collider. But $A \rightarrow X_3 \leftarrow X_4$ is an active path relative to $\{X_3\}$ since $X_3$ is a collider in the conditioning set $\{X_3\}$ .

\noindent \textbf{Faithfulness.} The distribution of the variables and $\g$ are faithful to each other means for all $X_i, X_j, \mathbf{C}$, $X_i$ is conditional independent with $X_j$ on $\mathbf{C}$ if and only if there exists no active path from $X_i$ to $X_j$ with relative to $\mathbf{C}$. Faithfulness is an important and common assumption in the research field of causality, which will be also used in this paper.

\subsection{Fairness and Disparity}
\label{sec:fairness}
We list some popular algorithmic fairness definitions as follows:

\noindent \textbf{Demographic Parity (DP)}~\cite{dwork2012fairness}. 
A prediction $\hat{Y}$ satisfies demographic parity if $P(\hat{Y}=1|A=1) = P(\hat{Y}=1|A=0)$.

\noindent \textbf{Equalized Odds}~\cite{hardt2016equality}. 
A prediction $\hat{Y}$ satisfies equalized opportunity if $P(\hat{Y}=1|Y=y,A=1) = P(\hat{Y}=1|Y=y,A=0), \forall y \in \{0,1\}$.

\noindent \textbf{Equalized Opportunity (EO)}~\cite{hardt2016equality}.
A prediction $\hat{Y}$ satisfies equalized opportunity if $P(\hat{Y}=1|Y=1,A=1) = P(\hat{Y}=1|Y=1,A=0)$.

\noindent \textbf{Accuracy Parity}~\cite{zafar2017fairness}
A prediction $\hat{Y}$ satisfies accuracy parity if $P(\hat{Y}=Y|A=1) = P(\hat{Y}=Y|A=0)$.

In practice, we can take the difference between the two sides of the equalities in the above definitions as a quantification measure for disparity. For example,
\begin{eqnarray}
	\Delta_{DP} &=& P(\hat{Y}=1|A=1) - P(\hat{Y}=1|A=0)\label{eq:disparity}\\
	\Delta_{EO} &=& P(\hat{Y}=1|Y=1,A=1) - P(\hat{Y}=1|Y=1,A=0)\label{eq:disparity-EO}
\end{eqnarray}
are two popular algorithm disparity measures used in fairness learning algorithms \cite{madras2018learning, song2019learning}. Here we follow the work of \cite{begley2020explainability} to use signed difference across groups to show which group is privileged. These fairness definitions are based on the (conditional) independence of the $\hat{Y}$ and the $A$, which can be determined using the causal graph \cite{baer2019fairness}. For example, according to the definitions in Section \ref{sec:causal_models}, $\hat{Y}$ is independent to $A$ if there is no active path between them, which can obtain demographic parity. In other words, any non-zero $\Delta_{DP}$ comes from the active paths linking $A$ and $\hat{Y}$. In the example in Figure \ref{fig:intro_examples}, active paths between $A$ and $\hat{Y}$ contain $A \leftarrow X_1 \rightarrow \hat{Y}$, $A \rightarrow X_2 \rightarrow \hat{Y}$ and $ A \rightarrow X_3 \rightarrow \hat{Y}$.

\subsection{Shapley Values}
\label{sec:shap}
Shapley values is a popular concept that has been used in model interpretation in recent years \cite{lundberg2017unified,aas2019explaining,frye2020asymmetric,heskes2020causal}. These methods typically decompose the prediction $f(\mathbf{x})$ for a given $\mathbf{x}$ as follows. 
\begin{equation}
	\label{eq:shap}
	f(\mathbf{x}) = \phi_{f}(0) + \sum\nolimits_{i=1}^{M} \phi_{f(\mathbf{x})}(i)
\end{equation}
\noindent where $\phi_{f(\mathbf{x})}(i)$ is the contribution of feature $X_i$ to $f(\mathbf{x})$. $\phi_{f}(0) = \mathbb{E}f(\mathbf{X})$ is the averaged prediction with the expectation over the observed data distribution $P(\mathbf{X})$. $\phi_{f(\mathbf{x})}(i)$ is referred as the Shapley value of feature $X_i$ for $f(\mathbf{x})$.

In order to calculate $\phi_{f(\mathbf{x})}(i)$, we firstly assume a sequential order $\pi$ for all variables in $\mathbf{X}$, such that $\pi(i)$ corresponds to the rank of $X_i$. The Shapley value of $X_i$ for $f(\mathbf{x})$ with respect to $\pi$ is
\begin{equation}
	\label{eq:shapley_fea}
	\phi^{\pi}_{f(\mathbf{x})}(i) = v_{f(\mathbf{x})}(\{X_j:\pi(j) \leq \pi(i)\}) - v_{f(\mathbf{x})}(\{X_j:\pi(j) < \pi(i)\})
\end{equation}
\noindent where $v_{f(\mathbf{x})}(\mathbf{S})$ represents the model's output on a selected coalition of features $\mathbf{X}_{\mathbf{S}}$ with a feature subset $\mathbf{S} \subset \{X_1,...,X_M\}$. $v_{f(\mathbf{x})}(\mathbf{S})$ must satisfy that $v_{f(\mathbf{x})}(\mathbf{S}) = f(\mathbf{x})$ when $\mathbf{S} = \mathbf{X}$ and $v_{f(\mathbf{x})}(\mathbf{S}) = \phi_{f}(0)$ when $\mathbf{S} = \emptyset$. With $\Pi$ denoting the set of all permutations of features and $w$ denoting a permutation weight satisfying $w(\pi) > 0$ and $\sum_{\pi \in \Pi} w(\pi) = 1$, we can calculate $\phi_{f(\mathbf{x})}(i)$ as
\begin{equation}
	\label{eq:shap_asys}
	\phi_{f(x)}(i) = \sum\nolimits_{\pi \in \Pi} w(\pi) \phi^{\pi}_{f(\mathbf{x})}(i)
\end{equation}
Different versions of Shapley values can be calculated based on different choices of $v_{f(\mathbf{x})}(\mathbf{S})$ and $w$. An obvious choice would be to take a uniform distribution $w(\pi) = \frac{1}{M!}$ and calculate $v_{f(\mathbf{x})}(\mathbf{S})$ as the expectation over the observed distribution of unselected features: $v_{f(\mathbf{x})}(\mathbf{S}) = \mathbb{E}_{\mathbf{X}'_{\bar{\mathbf{S}}}} [f(\mathbf{x}_\mathbf{S}, {\mathbf{x}'}_{\bar{\mathbf{S}}})]$ \cite{lundberg2017unified}. Aas {\em et al.} \cite{aas2019explaining} considered the correlation among $\mathbf{x}$ and propose the calculation of $v_{f(\mathbf{x})}(\mathbf{S})$ as $v_{f(\mathbf{x})}(\mathbf{S}) = \mathbb{E}_{\mathbf{X}'_{\bar{\mathbf{S}}}|\mathbf{X}_\mathbf{S}} [f(\mathbf{x}_{\mathbf{S}}, {\mathbf{x}'}_{\bar{\mathbf{S}}})]$. The resulting Shapley-values are referred to as {\em on-manifold} Shapley-values. 
Frye {\em et al.} \cite{frye2020asymmetric} further considered different choices of $w$. For example, one reasonable approach is to put weights only on those permutations which are consistent with known causal orderings:
\begin{equation}
	w(\pi) \propto \left\{ 
	\begin{array}{ll}
		1 & \text{if $\pi(i) < \pi(j)$ for any known} \\
		~ & \text{ancestor $X_i$ of descendant $X_j$} \\[5pt]
		0 & \text{otherwise}
	\end{array} \right.
\end{equation}

\subsection{Fairness Explanation}
There have been a few studies trying to derive explanations for model fairness, which can be categorized as either feature-based or path-specific explanation.

\subsubsection{Feature-based Explanation}
Lundberg \cite{lundberg2020fairness} and Begley {\em et al.} \cite{begley2020explainability} leveraged Shapley values defined in Eq.(\ref{eq:shap}) to attribute the feature contributions to $\Delta_{DP}$ in Eq.(\ref{eq:disparity}). Specifically, they proposed to use the group difference of $\phi_{f(x)}(i)$ to quantify the contribution of $X_i$ to $\Delta_{DP}$ as follows:
\begin{equation}
	\label{eq:dp_feat}
	\Phi_{f}(i) = \mathbb{E}_{\mathbf{X}|A= 1}[\phi_{f(\mathbf{x})}(i)] - \mathbb{E}_{\mathbf{X}|A=0}[\phi_{f(\mathbf{x})}(i)]
\end{equation}
\noindent Begley {\em et al.} \cite{begley2020explainability} has also extended this formulation to disparity measured on equalized opportunity as in Eq.(\ref{eq:disparity-EO}). However, decomposing model disparity into feature contributions ignores the causal structure of the features. 

\subsubsection{Path-Specific Explanations} 
There have been existing research studying fairness and the causal effects on the outcome by flipping the sensitive attribute value. They also studied the causal effects from particular paths, which are called path-specific effects \cite{kusner2017counterfactual,nabi2018fair,chiappa2019path,wu2019pc}. 
Intuitively, the path-specific effect can be viewed as quantification of path-specific contribution to model disparity. 
Existing research has only focused on causal paths (paths in which all edges are directed outwards $A$) so far, this may miss other sources of disparity. For example, in Figure \ref{fig:intro_examples}, $A \leftarrow X_1 \rightarrow \hat{Y}$ is an active path linking $A$ and $\hat{Y}$ and thus may contribute to model disparity, but it is not a causal path. Therefore, the sum of the path-specific effects considering only causal paths will not amount to the entire model disparity (e.g., measured by Eq.(\ref{eq:disparity})), which leads to incomplete explanations.

In this paper, we propose a novel algorithmic fairness explanation method called Fairness-Aware Causal paTh decompoSition (FACTS), which is introduced in the next section.

\section{Methodology}
In this section, we will introduce our framework with $\Delta_{DP}$ as the model disparity metric. We provide generalizations of our framework to other model disparity metrics in the appendix. 

Our framework is based on a causal graph $\g$, which could be obtained based on domain knowledge or learned from the training data with existing causal discovery algorithms \cite{spirtes2000causation}. According to Section \ref{sec:fairness}, active paths from $A$ to $\hat{Y}$ are sources of $\Delta_{DP}$. If there is no active path between $A$ and $\hat{Y}$, $\Delta_{DP} = 0$. If $\g$ is a DAG, we can identify all active paths based on the definition above and analyze their contributions to disparity. However, in reality $\g$ could be a PDAG, where the causal directions on some edges cannot be determined \cite{borboudakis2012incorporating,perkovic2020identifying}. This makes the problem more challenging.

In this section, we first propose a corresponding concept {\em potential active path} to represent the correlation relations between $A$ and $\hat{Y}$ under a PDAG, and each potential active path between $A$ and $\hat{Y}$ is referred to as a {\em fairness associated causal path} (FACT) in this paper. Then we propose an algorithm to extract all potential active paths from $A$ to $\hat{Y}$. Finally, we decompose $\Delta_{DP}$ as the sum of the contributions of these paths following the Shapley values strategy. 

\begin{figure}[tbp]
	\centering
	\includegraphics[width=3.2in]{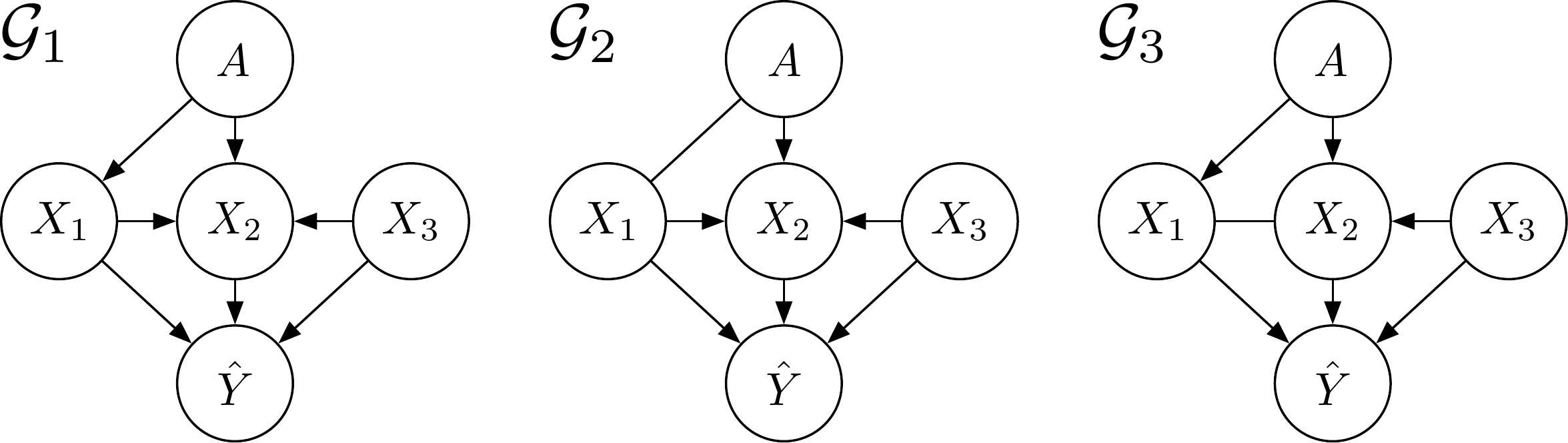}
	\caption{Three examples of PDAG. }
	\label{fig:example_pos_active}
% 	\vspace{-1.5em}
\end{figure}

\subsection{Potential Active Paths from $A$ to $\hat{Y}$}
Potential active paths on a PDAG are defined as follows:

\noindent \textbf{Potential Active Paths.} A \textit{path} $p$ in $\g$ is a potential active path if one of following properties is satisfied:
\begin{enumerate}
	\item All edges on $p$ are directed and $p$ satisfies the definition of active path in Section \ref{sec:causal_models}. 
	\item $p$ contains undirected edges. If we add arbitrary directions to all undirected edges (adjacent node pairs) on $p$, the resulting path is active. 
	\item $p$ contains undirected edges. Considering all possible directions of undirected edges in $\g$, there exists at least one situation to obtain a DAG $\g'$ which satisfies:  1) corresponding path obtained from $p$ in $\g'$ is active; 2) the conditional independence relationships among variables encoded in $\g'$ are consistent with those inferred from observational data. 
\end{enumerate}

We illustrate the above definition with examples in Figure \ref{fig:example_pos_active}. For path $A \rightarrow X_1 \rightarrow X_2 \rightarrow \hat{Y}$ in ${\g}_1$, since it is an active path by definition, it is a potential active path as well. As for $A - X_1 \rightarrow X_2 \rightarrow \hat{Y}$ in ${\g}_2$, since no matter what direction of $A-X_1$ is, the resulting path $A \rightarrow X_1 \rightarrow X_2 \rightarrow \hat{Y}$ or $A \leftarrow X_1 \rightarrow X_2 \rightarrow \hat{Y}$ is active. Thus $A - X_1 \rightarrow X_2 \rightarrow \hat{Y}$ in ${\g}_2$ is a potential active path. For $A \rightarrow X_1 - X_2 \rightarrow \hat{Y}$ in ${\g}_3$, suppose the conditional independence relation obtained from data is $A \nind X_1 | X_2, A \nind X_2 | X_1, X_1 \nind X_2 | A$. Consider the directions of all undirected edges to be $X_1 \rightarrow X_2$, the corresponding path $A \rightarrow X_1 \rightarrow X_2 \rightarrow \hat{Y}$ in ${\g}_3$ under this case is active and the relationships among variables is consistent with the conditional independence relation observed from the data. So $A \rightarrow X_1 - X_2 \rightarrow \hat{Y}$ is a potential active path.

When $\g$ is DAG, the potential active path is equivalent to the active path. The potential active paths satisfy the following property:
\begin{proposition}
	\label{prop1}
	If there is no potential active path between two variables on $\g$, then the two variables are independent.
\end{proposition}

This proposition shows the importance of potential active paths between $A$ and $\hat{Y}$ when we consider $\Delta_{DP}$. Since if there is no potential active path between $A$ and $\hat{Y}$, $A \ind \hat{Y}$ and $\Delta_{DP} = 0$.

In ordering to search potential active paths more efficiently, we have the following proposition:
\begin{proposition}
	\label{prop2}
	If $p$ is a potential active path in $\g$, then every subpath of $p$ is also a potential active path.
\end{proposition}

The proofs of propositions are provided in the appendix. Based on this proposition, we propose an algorithm to search all potential active paths from $A$ to $\hat{Y}$ as demonstrated in Algorithm \ref{alg:search}.

\begin{algorithm}
	\raggedright
	\caption{Search Potential Active Paths from $A$ to $\hat{Y}$
		\label{alg:search}
	}
	\textbf{Input:} A PDAG $\g$, Dataset $\{A, X\}$\\
	\textbf{Output:} A set of potential active paths from $A$ to $\hat{Y}$: $\mathbf{P}$, A set of features involved in $\mathbf{P}$: $\mathbf{X}\left(\mathbf{P}\right)$\\
	\textbf{Initialization:} \\
	$\mathbf{P}_{A\rightarrow} = \{$Path directly connects $A$ and $X_i$ : $X_i \in \Adj(A)\}$, $\mathbf{P} = \emptyset$\\
	$\mathbf{X}\left(\mathbf{P}\right) = \Adj(A)$ \Comment{$\Adj(A)$ is the set of nodes adjacent to $A$ on $\g$.}\\
	(Here $\mathbf{P}_{A\rightarrow}$ is a temporary set to store the potential active paths from $A$ during searching)\\
	\begin{algorithmic}[1]
		\While{$\mathbf{P}_{A\rightarrow} \neq \emptyset$}
		\State Let $p \in \mathbf{P}_{A\rightarrow}$
		\State Remove $p$ from $\mathbf{P}_{A\rightarrow}$
		\For{$X \gets \Adj(p[-1])$} \Comment{$p[-1]$ means the last node of $p$}
		\State $p' = p + X$
		\If{$p'$ is a potential active path by definition}
		\State $\mathbf{X}\left(\mathbf{P}\right) = \mathbf{X}\left(\mathbf{P}\right) \cup \{X\}$
		\If{$X$ is not $\hat{Y}$}
		\State $\mathbf{P}_{A\rightarrow} = \mathbf{P}_{A\rightarrow} \cup \{p'\}$
		\Else
		\State $\mathbf{P} = \mathbf{P} \cup \{p'\}$ 
		\EndIf
		\EndIf
		\EndFor
		\EndWhile
	\end{algorithmic}
\end{algorithm}

Table \ref{table:notation} summarizes the notations that will be needed for the follow up presentations. As an example in ${\g}_1$ of Figure \ref{fig:example_pos_active}, $\mathbf{P} = \{A \rightarrow X_1 \rightarrow \hat{Y}, A \rightarrow X_2 \rightarrow \hat{Y}, A \rightarrow X_1 \rightarrow X_2 \rightarrow \hat{Y} \}$, $\mathbf{X}\left(\mathbf{P}\right) = \{X_1, X_2\}, \bar{\mathbf{X}}\left(\mathbf{P}\right) = \{X_3\}$. We will have $\bar{\mathbf{X}}\left(\mathbf{P}\right) \ind A$.

\begin{table}[ht]
	\caption{Important Math Notations}
	\begin{center}  
		\begin{tabular}{c|c}
			\toprule
			\textbf{Notation} & \textbf{Meaning} \\
			\midrule
			$\mathbf{P}$ & The set of potential active paths from $A$ to $\hat{Y}$ \\
			$\mathbf{X}\left(\mathbf{P}\right)$ & The set of features involved in $\mathbf{P}$  \\
			$\bar{\mathbf{X}}\left(\mathbf{P}\right)$ & Set of features not involved in $\mathbf{P}$  \\
			${\phi}_{f(\mathbf{x})}(p_i)$ & Contribution of path $p_i$ to $f(\mathbf{x})$ \\
			${\Phi}(p_i)$ & Contribution of path $p_i$ to $\Delta_{DP}$ \\
			\bottomrule
		\end{tabular}  
	\end{center} 
% 	\vspace{-0.5em}
	\label{table:notation}
\end{table}

\noindent \textbf{Ordering Relationships with respect to $A$.} For two nodes $X_i$ and $X_j$ in $\mathbf{X}\left(\mathbf{P}\right)$, $X_i$ is defined to be prior to $X_j$ with respect to $A$ if:
\begin{enumerate}
	\item There exist at least one potential active paths from $A$ to $\hat{Y}$ that both $X_i$ and $X_j$ are on these paths;
	\item For all potential active paths from $A$ to $\hat{Y}$ containing both $X_i$ and $X_j$, $X_i$ precedes $X_j$ on them.
\end{enumerate}
\noindent In the following, we write such relationship as $X_i \succ_{A} X_j$. We also define: $\forall X_i, A \succ_{A} X_i, X_i \succ_{A} \hat{Y}$. 

If $X_j \succ_{A} X_i$ and $X_j$ is adjacent to $X_i$, we call $X_i$ is a {\em direct successor} of $X_j$ with respect to $A$, and $X_j$ is a {\em direct predecessor} of $X_i$ with respect to $A$. The set of all direct successors of $X_i$ w.r.t. $A$ is denoted by $\Ds_{A}(X_i)$. The set of all direct predecessors of $X_i$  w.r.t. $A$ is denoted by $\Dp_{A}(X_i)$. For ${\g}_1$ in Figure \ref{fig:example_pos_active}, we have $\Dp_{A}(X_1) = \{A\}, \Dp_{A}(X_2) = \{A, X_1\}$. We will write $\Dp_{A}(X_i)$ as ${\Dp}_i$ and $\Ds_{A}(X_i)$ as ${\Ds}_i$ for simplicity in the following. 

\noindent \textbf{Completely Ordered with respect to $A$.} $\mathbf{X}\left(\mathbf{P}\right)$ is defined to be completely ordered with respect to $A$ on $\g$ if $\forall X_i, X_j \in \mathbf{X}\left(\mathbf{P}\right), X_i \in \Adj(X_j)$, one of them must be the direct successor or predecessor of the other. This can also be written as $\forall X_i, X_j \in \mathbf{X}\left(\mathbf{P}\right), X_j \in \Adj(X_i)$, we have $X_j \in {\Dp}_i \cup {\Ds}_i$. Considering 
a special case where $\g$ is a DAG, we have the following proposition. 
\begin{proposition}
	\label{prop3}
	If $\g$ is a DAG and no other node is the parent of $A$, then $\mathbf{X}\left(\mathbf{P}\right)$ is completely ordered with respect to $A$ on $\g$.
\end{proposition}
\noindent where the condition is sufficient but not necessary for completely ordering w.r.t. $A$. For example, as in ${\g}_2$ of Figure \ref{fig:example_pos_active}, even though the edge $X_{1}-A$ is undirected, we can still have that $\mathbf{X}\left(\mathbf{P}\right)$ is completely ordered with respect to $A$.

\subsection{FACT Decomposition of Model Disparity}
\label{sec:path_shap_dis}
In this subsection, we propose an algorithm to quantitatively attribute the model disparity $\Delta_{DP}$ to individual FACTs. We first present the algorithm for the case when $\mathbf{X}\left(\mathbf{P}\right)$ is completely ordered with respect to $A$ on $\g$. The extension of our algorithm to the scenario when the completely-ordered condition does not hold is introduced in Section \ref{sec:algext}. 

Specifically, our algorithm is based on the Shapley values strategy in Section \ref{sec:shap}, and our goal is to decompose $\Delta_{DP}$ as the sum of contributions from the paths in $\mathbf{P}$ as $\Delta_{DP}=\sum_{p_i\in\mathbf{P}}\Phi_{f}(p_i)$, and
\begin{equation}
	\label{eq:dp_fact}
	\Phi_{f}(p_i) = \mathbb{E}_{\mathbf{X}|A= 1}[\phi_{f(\mathbf{x})}(p_i)] - \mathbb{E}_{\mathbf{X}|A=0}[\phi_{f(\mathbf{x})}(p_i)]
\end{equation}
where $\phi_{f(\mathbf{x})}(p_i)$ is the Shapley value of FACT $p_i$. Following Eq.(\ref{eq:shapley_fea}), we can define $\phi_{f(\mathbf{x})}(p_i)$ as
\begin{equation}
	\label{eq:calculation}
	\begin{aligned}
		& \phi_{f(\mathbf{x})}(p_i) = \\
		& \sum_{\pi \in \Pi} \frac{v_{f(\mathbf{x})}(\{p_j: {\pi}(p_j) \leq {\pi}(p_i)\}) - v_{f(\mathbf{x})}(\{p_j: {\pi}(p_j) < {\pi}(p_i)\})}{|\Pi|}
	\end{aligned}
\end{equation}
where $\pi$ is a permutation function for all FACTs, and $\Pi$ is the collection of all permutations. $v_{f(\mathbf{x})}(\pathsubset)$ is a value function defined on a set of FACTs $\pathsubset\subset\mathbf{P}$. 
In order to appropriately define $v_{f(\mathbf{x})}(\pathsubset)$, we need to leverage the causality structure encoded in $\mathbf{P}$. In particular, $\mathbf{P}$ can be seen as a system which transfer the information of $A$ downstreams and finally affect the value of $\hat{Y}$. So an intuitive idea is to formulate this system as a calculation process started from $A$, passing through FACTs and finally get the model prediction $f(\mathbf{x})$. During the inference process, each $X_i \in \mathbf{X}\left(\mathbf{P}\right)$ can be estimated as
\begin{equation}
	\label{eq:regression}
	X_i = g_{X_i}\left(\Dp_{i}, \bar{\mathbf{X}}\left(\mathbf{P}\right), E_i\right)
\end{equation}
where $g_{X_i}$ is the regression link function, $\Dp_{i}$ is the set of predecessors of $X_i$ in $\mathbf{P}$, $\bar{\mathbf{X}}\left(\mathbf{P}\right)$ is the set of feature variables that are not in $\mathbf{X}(\mathbf{P})$, $E_i$ is the random regression error. We assume $\{E_i\}$ are mutually independent and each $E_i$ is independent of $\Dp_{i}$ and $\bar{\mathbf{X}}\left(\mathbf{P}\right)$. Hyvarinen {\em et al.} \cite{hyvarinen1999nonlinear} proved that we can always construct such $\{g_{X_i}\}$ and $\{E_i\}$. In the following we present the calculation process with a concrete example.

Considering ${\g}_1$ in Figure \ref{fig:example_pos_active}, we have  $\mathbf{X}\left(\mathbf{P}\right) = \{X_1, X_2\}, \bar{\mathbf{X}}\left(\mathbf{P}\right) = \{X_3\}$, $\mathbf{P} = \{A \rightarrow X_1 \rightarrow \hat{Y}, A \rightarrow X_2 \rightarrow \hat{Y}, A \rightarrow X_1 \rightarrow X_2 \rightarrow \hat{Y} \}$. According to Eq.(\ref{eq:regression}), we also have $X_1 = g_{X_1}(A,X_3,E_1)$, $\mathbf{Dp}_2=X_1$ and $X_2 = g_{X_2}(A,X_1,X_3,E_2)$. For calculating $v_{f(\mathbf{x})}(\pathsubset)$, we first consider two extreme cases.
\begin{itemize}
	\item $\pathsubset = \mathbf{P}$. In this case, the actual value of $A$ is visible to all FACTs, which makes $v_{f(\mathbf{x})}(\pathsubset)=f(\mathbf{x})$.
	
	\item $\pathsubset = \emptyset$. In this case, the actual value of $A$ is visible to none of the FACTs. We can sample $A$ from its marginal distribution $a' \sim P(A)$ and calculate $X_i$ under $A = a'$, denoted as $x_i(a')$, then we calculate $v_{f(\mathbf{x})}(\pathsubset)$ as
	\begin{equation}
		\begin{aligned}
			& x_1(a') = g_{X_1}(a',x_3,e_1) \\
			& x_2(a') = g_{X_2}(a',x_1(a'),x_3,e_2) \\
			&v_{f(\mathbf{x})}(\mathbf{T}) = \mathbb{E}_{a' \sim P(A)} f(x_1(a'), x_2(a'), x_3)
		\end{aligned}
	\end{equation}
\end{itemize}

The case of $\pathsubset\subset\mathbf{P}$ is more complicated. We denote $x_i(a'|\pathsubset)$ as the value of $X_i$ with $\pathsubset$. In such process, the values of $A$ which pass through $\pathsubset$ will be set to $a$, while those passing through $\mathbf{P} \setminus \pathsubset$ will be set to $a' \sim P(A)$. Considering the situation $\pathsubset = \{A \rightarrow X_1 \rightarrow X_2 \rightarrow \hat{Y}\}$ in the example, we need to transfer the information of $A = a$ along the path $A \rightarrow X_1 \rightarrow X_2 \rightarrow \hat{Y}$ but block this information and use a random sample $a'$ along other paths $A \rightarrow X_1 \rightarrow \hat{Y}$ and $A \rightarrow X_2 \rightarrow \hat{Y}$: 
\begin{equation}
	\begin{aligned}
		& x_1(a'|\mathbf{T}) = g_{X_1}(a',x_3,e_1) \\
		& x_2(a'|\mathbf{T}) = g_{X_2}(a',x_1,x_3,e_2) \\
		& v_{f(\mathbf{x})}(\pathsubset) = \mathbb{E}_{a' \sim P(A)} f(x_1(a'|\mathbf{T}), x_2(a'|\mathbf{T}), x_3)
	\end{aligned}
\end{equation}
In this way, we can calculate $\phi_{f(\mathbf{x})}(p_i)$ as in Eq.(\ref{eq:calculation}).

In practice, we can choose $f(\mathbf{x})$ as the probability to predict $\mathbf{x}$ to be positive or the binary decision with a threshold on the probability. In the latter case, the decomposed Shapley values of $\Delta_{DP}$ on FACTs satisfy the following properties (detailed proofs are provided in the appendix).

\begin{enumerate}
	\item (Efficiency) $\sum_{p_i \in \mathbf{P}} \Phi_{f}(p_i) = \Delta_{DP}$
	\item (Linearity) $\Phi_{\alpha f_1 + \beta f_2}(i) = \alpha \Phi_{f_1}(i) + \beta \Phi_{f_2}(i)$;
	\item (Nullity) $\Phi_{f}(p_i)$ = 0 when $v_{f(\mathbf{x})}(\pathsubset \cup \{p_i\}) = v_{f(\mathbf{x})}(\pathsubset),~\forall \mathbf{x}, \pathsubset \subset \mathbf{P} \backslash p_i$
\end{enumerate}

Our path explanations can also be aggregated to generate feature-level explanations. To obtain the contribution of feature $X_i$ to $\Delta_{DP}$, we can sum the path contributions for all paths ended with "$\dots X_i \rightarrow \hat{Y}$". 

\subsection{Algorithm Implementation}
\label{sec:algext}
When $\g$ is not completely ordered, some of FACTs could be contradictory to each other. Considering ${\g}_3$ in Figure \ref{fig:example_pos_active}, we have $\mathbf{P} = \{A \rightarrow X_1 \rightarrow \hat{Y}, A \rightarrow X_2 \rightarrow \hat{Y}, A \rightarrow X_1 - X_2 \rightarrow \hat{Y}, A \rightarrow X_2 - X_1 \rightarrow \hat{Y}\}$. 
In this case, $A \rightarrow X_2 - X_1 \rightarrow \hat{Y}$ and $A \rightarrow X_1 - X_2 \rightarrow \hat{Y}$ cannot be active simultaneously.

One solution is to consider all orientation possibilities of undirected edges. For example, the direction of $X_1 - X_2$ in ${\g}_3$ could be either $X_1 \leftarrow X_2$ or $X_1 \rightarrow X_2$. We can study each situation respectively and then summarize them.  However, this makes the exploration space potentially huge (suppose we have $N_{u}$ undirected edges, then we can have $2^{N_{u}}$ orientation possibilities to explore).

We propose to solve this problem by grouping the adjacent feature variables which cause the inconsistency problem. In the example of ${\g}_3$ in Figure \ref{fig:example_pos_active}, if we group $X_1$ and $X_2$ as ${\chi}_1 = \{X_1, X_2\}$ and treat it as a single variable, then ${\g}_3$ satisfies the completely-ordered condition. Therefore, we propose to first obtain a partition of $\mathbf{X}\left(\mathbf{P}\right)$ as $\chi\left(\mathbf{P}\right) = \{{\chi}_1, \dots, {\chi}_K\}$ so that $\chi\left(\mathbf{P}\right)$ is completely ordered with respect to $A$ with these grouped variables $\{{\chi}_k\}_{k=1}^K$. Then we can calculate the contributions of the paths with group-level variables. The concrete algorithm implementation steps are as follows.

\noindent \emph{\underline{Step 1: Identify Potential Active Paths from $A$ to $\hat{Y}$}}. With $\g$, we find all potential active paths from $A$ to $\hat{Y}$ as FACTs with Algorithm \ref{alg:search}. 

\noindent \emph{\underline{Step 2: Generate Groups Completely Ordered With Respect To $A$}}. We group the feature variables so that the grouped variables are completely ordered with respect to $A$ with Algorithm \ref{alg:grouping}.

\noindent \emph{\underline{Step 3: Calculate Path Contributions to $\Delta_{DP}$}}. For each group level variable ${\chi}_i$, we learn a prediction link function $g_{{\chi}_i}$ and obtain the error term $\mathbf{E}_{{\chi}_i}$ (both $g_{{\chi}_i}$ and $\mathbf{E}_{{\chi}_i}$ are multi-dimensional, with each dimension corresponding to an individual feature variable in ${\chi}_i$). Finally we calculate the path contribution on the group-level with the procedure in Section \ref{sec:path_shap_dis}.

\begin{algorithm}
	\raggedright
	\caption{Generate Groups Completely Ordered w.r.t $A$\label{alg:grouping}}
	\textbf{Input:} A PDAG $\g$, $\mathbf{P}$, $\mathbf{X}\left(\mathbf{P}\right)$\\
	\textbf{Output:} A division of $\mathbf{X}\left(\mathbf{P}\right)$ into groups ${\chi}_1,\dots,{\chi}_K$, A set of active paths on group-level $\mathbf{P}'$\\
	\textbf{Initialization:} \\
	Create groups $\{{\chi}_{i}\}$, each ${\chi}_{i}$ containing $X_i \in \mathbf{X}\left(\mathbf{P}\right)$ and path set $\mathbf{P}'$ by replacing $X_i$ with ${\chi}_i$ for all paths in $\mathbf{P}$\\
	
	\begin{algorithmic}[1]
		\While{$\{{\chi}_{i}\}$ is not completely ordered w.r.t A}
		\State Let ${\chi}_{i}: {\chi}_{j} \in \Adj({\chi}_{i}), {\chi}_{j} \notin \Dp({\chi}_{i}) \cup \Ds({\chi}_{i})$
		\State Merge ${\chi}_{i}$ and all ${\chi}_{j} \in \Adj({\chi}_{i}) \setminus \Dp({\chi}_{i}) \cup \Ds({\chi}_{i})$, get ${\chi}^{*}$
		\State \Call{UpdateRelation}{${\chi}^{*},\mathbf{P}'$}
		\EndWhile
		
		\Function{UpdateRelation}{${\chi}^{*},\mathbf{P}'$}
		\For{$p \gets \mathbf{P}'$}
		\If{$p$ contains node ${\chi}_i:{\chi}_i
			\subset {\chi}^{*}$}
		\State Replace all ${\chi}_i$ (${\chi}_i
		\subset {\chi}^{*}$) by ${\chi}^{*}$ on $p$
		\State Merge repeated ${\chi}^{*}$ on $p$
		\EndIf 
		\EndFor
		\State Remove repeated paths in $\mathbf{P}'$
		\EndFunction
	\end{algorithmic}
	
\end{algorithm}

\subsection{Fair Learning Through FACT Selection}
\label{sec:fair_learning}
With the FACTs based model disparity decomposition approach, we can obtain the quantitative contribution of each FACT to the model disparity. At the same time, if we also consider the model utility, then we can select the paths with high model utility and low model disparity contributions when building the model.

Without the loss of generality, we assume the outcome variable $Y\in\{0,1\}$ and the prediction model $f(\mathbf{x})$ can return the prediction of $\mathbf{x}$ belonging to 1, then the utility of $f$ can be estimated as
\begin{equation}
	\mathcal{U}(f)=\mathbb{E}_{\mathbf{X},Y}[y f(\mathbf{x}) + (1-y)(1-f(\mathbf{x}))]
\end{equation}

For a given data sample $\{\mathbf{x},y\}$, we can calculate the specific model utility of this sample as
\begin{equation}
	\mathcal{U}(f(\mathbf{x}))=y f(\mathbf{x}) + (1-y)(1-f(\mathbf{x}))
\end{equation}

We can decompose $\mathcal{U}(f(\mathbf{x}))$ as the the contributions of FACTs in $\mathbf{P}$. Similar to Eq.(\ref{eq:calculation}), we define $\psi_{f(\mathbf{x}), y}(p_i)$ as
\begin{small}
	\begin{equation}
		\label{eq:calculation_acc}
		\begin{aligned}
			& \psi_{f(\mathbf{x}), y}(p_i) = \\
			& \sum_{\pi \in \Pi} \frac{v_{f_y(\mathbf{x})}(\{p_j: {\pi}(p_j) \leq {\pi}(p_i)\}) - v_{f_y(\mathbf{x})}(\{p_j: {\pi}(p_j) < {\pi}(p_i)\})}{|\Pi|}
		\end{aligned}
	\end{equation}
\end{small}
\noindent where $f_y(\mathbf{x}) = f(\mathbf{x})$ if $y = 1$ and $f_y(\mathbf{x}) = 1 - f(\mathbf{x})$ otherwise. In this way, the contribution of $p_i$ to $\mathcal{U}(f)$ is
\begin{equation}
	\Psi_{f}(p_i) = \mathbb{E}_{\mathbf{X},Y}[\psi_{f(\mathbf{x}),y}(p_i)]
\end{equation}

Thus $\mathcal{U}(f)$ can be decomposed as
\begin{equation}
	\Psi_{f}(\emptyset) + \sum\nolimits_{p_i \in \mathbf{P}} \Psi_{f}(p_i)
\end{equation}
where $\Psi_{f}(\emptyset) = \mathbb{E}_{\mathbf{X},Y}[ v_{f_y(\mathbf{x})}(\emptyset)]$.

With the decomposition of $\mathcal{U}(f)$ and $\Delta_{DP}$, we can construct an interpretable fair learning algorithm to achieve trade-off between accuracy and fairness through FACT selection. Specifically, our goal is to select a set of paths $\pathsubset^*\subset\mathbf{P}$ by minimizing the objective: 
\begin{equation}
	\label{eq:loss}
	\mathcal{L}(\pathsubset) =  -\sum\nolimits_{p_i \in \pathsubset} \Psi_{f}(p_i) + \lambda |\sum\nolimits_{p_i \in \pathsubset} \Phi_{f}(p_i)| 
\end{equation}

We propose a greedy algorithm shown in Algorithm \ref{alg:fl_searching} to solve this problem by iteratively removing the edges in $\mathbf{P}$ to get $\pathsubset^{*}$. After we obtain $\pathsubset^{*}$, $v_{f(\mathbf{x})}(\pathsubset^{*})$ is used as the new prediction result. 

\begin{algorithm}
	\raggedright
	\caption{Select Paths to Achieve Better Trade-off Between Accuracy and $\Delta_{DP}$\label{alg:fl_searching}}
	\textbf{Input:} $\mathbf{P}$, $\{\Phi_{f}(p_i):p_i \in \mathbf{P}\}$,  $\{\Psi_{f}(p_i):p_i \in \mathbf{P}\}$, $\lambda$\\
	\textbf{Output:} $\pathsubset^{*}$\\
	\textbf{Initialization:} $\pathsubset = \mathbf{P}$\\
	\begin{algorithmic}[1]
		\While{$\pathsubset \neq \emptyset$}
		\State Let $p^{*}= \mathop{\arg\min} \limits_{p \in \pathsubset} \mathcal{L}(\pathsubset \setminus \{p\})$
		\State Remove $p^{*}$ from $\pathsubset$
		\If{$\mathcal{L}(\pathsubset)$ < $\mathcal{L}(\pathsubset^{*})$}
		\State $\pathsubset^{*} = \pathsubset$
		\EndIf
		\EndWhile
	\end{algorithmic}
\end{algorithm}

\section{Experiments}
\subsection{Datasets}
\emph{\underline{Synthetic Data}}: We created a dataset with $10$ features under a DAG $\g$. The feature variables and the sensitive attribute $A$ are randomly connected by directed edges. We generated the data in two different settings: 1) $S_1$: the relation between features and outcome is linear; 2) $S_2$: the relation between features and outcome is non-linear. The detailed data generation process is described in the appendix. 

\noindent \emph{\underline{Adult}} \cite{lichman2013uci}: The Adult dataset consists of 48,842 samples. The task is to predict whether one's annual income is greater than 50K. We consider gender (male, female) as sensitive attribute and age, nationality, marital status, level of education, working class and hours per week as feature variables similar to \cite{zhang2018causal}. We set $Y=1$ for $Income \geq 50K$ and $A=1$ for male.

\begin{table}[thbp]
    \vspace{-0.5em}
	\caption{Normalized root-mean-square error of estimated path contributions.}
	\small
	\begin{tabular}{c|cc|cc}
		\toprule
		\multicolumn{1}{c}{} & \multicolumn{2}{c}{MLP} & 
		\multicolumn{2}{c}{Xgboost} \\
		\hline
		&  PSE & Ours & PSE & Ours \\
		\hline
		$S_1$ & 0.09 ($\pm$ 0.03) & {\bf 0.06} ($\pm$ 0.02) & 0.07 ($\pm$ 0.02) & 0.07 ($\pm$ 0.01) \\
		$S_2$ & 0.14 ($\pm$ 0.03) & {\bf 0.10} ($\pm$ 0.03) & 0.20 ($\pm$ 0.07) & {\bf 0.10} ($\pm$ 0.07)\\
		\bottomrule
	\end{tabular}
	\label{tab:est_dp_ctb}
    \vspace{-1.0em}
\end{table}

\begin{table}[thbp]
    \vspace{-0.5em}
	\caption{Normalized absolute difference of the summation of path contributions and $\Delta_{DP}$. If the difference equals 0, the $\Delta_{DP}$ can be completely explained by the contributions.}
	\small
	\begin{tabular}{c|cc|cc}
		\toprule
		\multicolumn{1}{c}{} & \multicolumn{2}{c}{MLP} & 
		\multicolumn{2}{c}{Xgboost} \\
		\hline
		&  PSE & Ours & PSE & Ours \\
		\hline
		$S_1$ & 0.05 ($\pm$ 0.03) & {\bf 0.01} ($\pm$ 0.01) &  0.04 ($\pm$ 0.02) & {\bf 0.01} ($\pm$ 0.01) \\
		$S_2$ & 0.09 ($\pm$ 0.04) & {\bf 0.01} ($\pm$ 0.01) & 0.05 ($\pm$ 0.03) & {\bf 0.01} ($\pm$ 0.01)\\
		Adult & 0.23( $\pm$ 0.04) & {\bf 0.07} ($\pm$ 0.01) & 0.29($\pm$ 0.02) & {\bf 0.10} ($\pm$ 0.01)\\
		COMPAS & 0.59( $\pm$ 0.04) & {\bf 0.06} ($\pm$ 0.03) & 0.66($\pm$ 0.10) & {\bf 0.15} ($\pm$ 0.07)\\
		\bottomrule
	\end{tabular}
	\label{tab:eff_dp}
	\vspace{-0.5em}
\end{table}

\noindent \emph{\underline{COMPAS}} \cite{angwin2016machine}: The dataset contains 6,172 samples and the goal is to predict whether a defendant will recidivate within two years or no ($Y=1$ for non-recidivism). Race is the sensitive attribute ($A=1$ for white people) and we choose 7 other attributes including age, gender, number of prior crimes, triple of numbers of juvenile felonies/juvenile misdemeanors/other juvenile convictions, original charge degree.

\noindent \emph{\underline{Nutrition}} (National Health and Nutrition Examination Survey) \cite{cox1998plan}: This dataset consists of $14,704$ individuals with $20$ demographic features and laboratory measurements. The target is to predict 15-year survival. We follow the data preprocessing procedures in \cite{wang2021shapley}. Race (white, non-white) is selected as the sensitive attribute. We set $Y=1$ for 15-year survival and $A=1$ for white people.

\subsection{Experimental Setting}
We evaluate \methodname~from two aspects: 1) path explanations for $\Delta_{DP}$; 2) fair learning through FACT selection. Here are some general settings to train model $f$ and learn the FACTs.

\noindent \textbf{Model Training}: We train $f$ using a $70\%/30\%$ train/test split. We randomly split data with this ratio and run experiments 5 times. The average result is reported. The hyper-parameters of the model are tuned by cross-validation. When we calculate path explanations for $\Delta_{DP}$, we implement $f$ as MultiLayer Perceptron (MLP) or Xgboost \cite{chen2016xgboost} and report the results respectively.

\begin{figure}
	\centering
	\subfigure[]{
		\includegraphics[width=0.22\columnwidth]{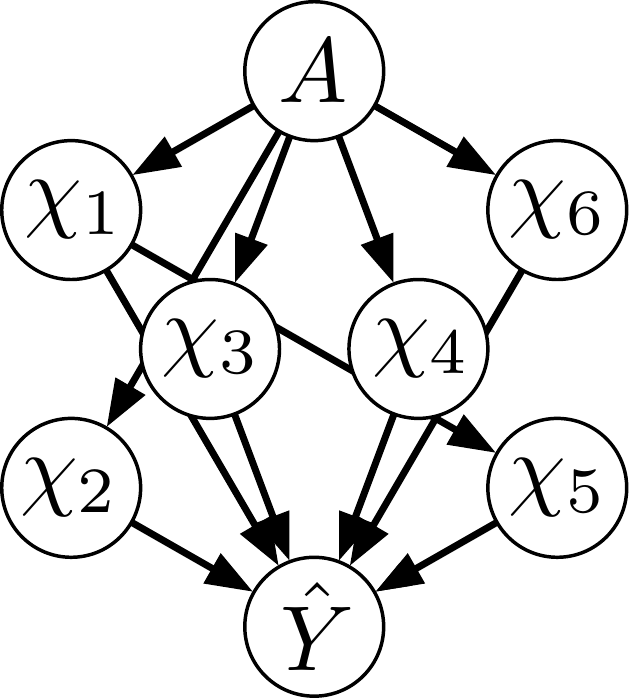}
		\label{fig:nutrition_graph}
	}
	\subfigure[]{
		\includegraphics[width=0.66\columnwidth]{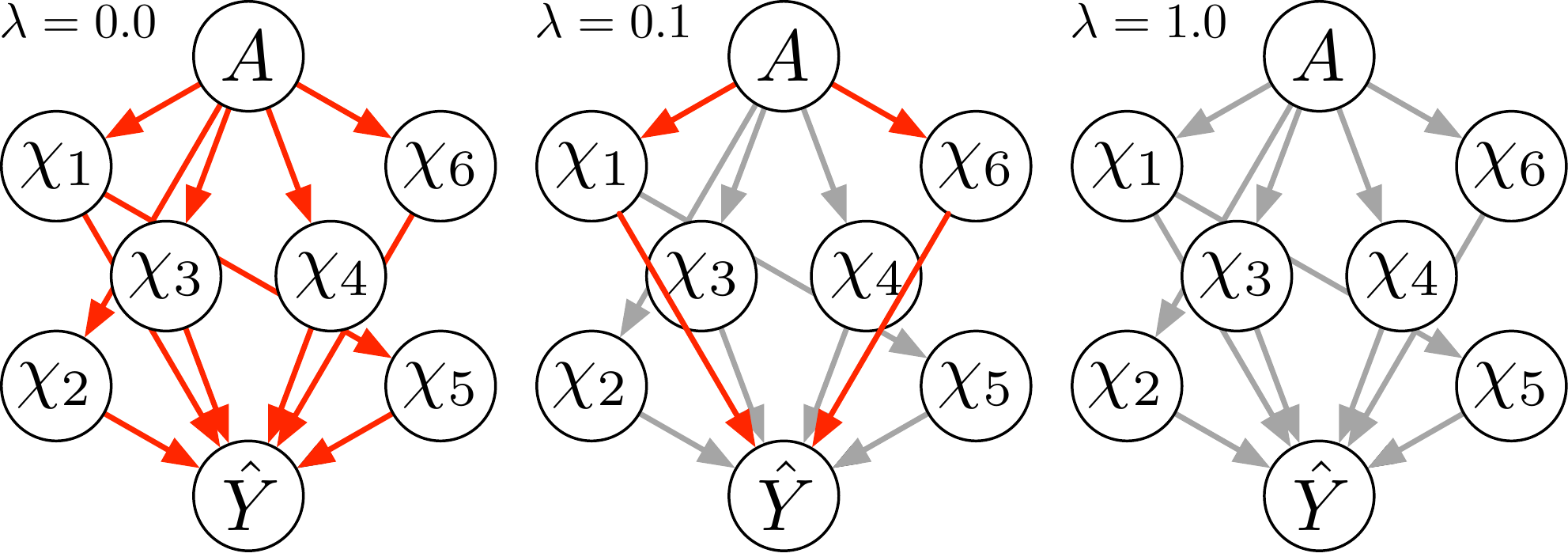}
		\label{fig:fl_path_selection}
	}
	
	\caption{(a): PDAG to explain $\Delta_{DP}$ on Nutrition, showing top 7 paths contributing to $\Delta_{DP}$. The meanings of nodes are in Table \ref{tab:nutrition_feature_contribution}. (b): The set of selected FACTs (${\pathsubset}^{*}$) under different values of $\lambda$ (larger $\lambda$ means stronger fairness constraints) on Nutrition dataset. Red paths indicates the paths in ${\pathsubset}^{*}$.}
% 	\vspace{-0.5em}
\end{figure}

\begin{table}[t]
	\caption{Feature explanations on Nutrition. Each column shows the contributions of features to $\Delta_{DP}$ with ISV/ASV.} 
	\begin{center}
	    \small
		\begin{tabular}{ccc}
			\toprule
			{Features} & ISV & ASV \\
			\midrule
			Poverty Idx, Food Program (${\chi}_1$) & \cellcolor{red!90} 0.0318 & \cellcolor{red!100} 0.0355\\
			Blood pressure (${\chi}_2$) & \cellcolor{red!40} 0.0141 & \cellcolor{red!35} 0.0129\\
			Serum magnesium (${\chi}_3$) & \cellcolor{red!24} 0.0076 & \cellcolor{red!27} 0.0079\\
			Blood protein (${\chi}_4$) & \cellcolor{red!18} 0.0064 & \cellcolor{red!25} 0.0075\\
			Sedimentation rate (${\chi}_5$) & \cellcolor{red!30} 0.0099 & \cellcolor{red!25} 0.0061\\
			White blood cells, Red blood cells (${\chi}_6$) & \cellcolor{blue!90} -0.0077 & \cellcolor{blue!100} -0.0082 \\
			\bottomrule
		\end{tabular}
	\end{center}
	\label{tab:nutrition_feature_contribution}
	\vspace{-0.5em}
\end{table}

\begin{table}[t]
	\caption{Path explanations on Nutrition dataset. The first column shows the contributions of paths to $\Delta_{DP}$~($\Phi_{f}(p_i)$), the second column shows the contributions to utility~($\Psi_{f}(p_i)$).}  
	\begin{center}  
	    \small
		\begin{tabular}{ccc}
			\toprule
			Paths & $\Phi_{f}(p_i)$ & $\Psi_{f}(p_i)$ \\
			\midrule
			$A \rightarrow {\chi}_1 \rightarrow \hat{Y}$ & \cellcolor{red!90} 0.0324 & 0.0039 \\
			$A \rightarrow {\chi}_2 \rightarrow \hat{Y}$& \cellcolor{red!40} 0.0126 & 0.0009\\
			$A \rightarrow {\chi}_3 \rightarrow \hat{Y}$ & \cellcolor{red!26} 0.0081 & 0.0006\\
			$A \rightarrow {\chi}_4 \rightarrow \hat{Y}$ & \cellcolor{red!25} 0.0077 & 0.0032\\
			$A \rightarrow {\chi}_5 \rightarrow \hat{Y}$ & \cellcolor{red!20} 0.0060 & 0.0006\\
			$A \rightarrow  {\chi}_1 \rightarrow {\chi}_5 \rightarrow \hat{Y}$ & \cellcolor{red!10} 0.0031 & 0.0009\\
			$A \rightarrow {\chi}_6 \rightarrow \hat{Y}$ & \cellcolor{blue!100} -0.0082 & < 0.0001\\
			\bottomrule
		\end{tabular}
	\end{center}
	\label{tab:nutrition_path_contribution}
	\vspace{-0.5em}
\end{table}

\noindent \textbf{Causal Discovery}: For the synthetic dataset, we directly use the ground-truth causal graph. For real-world datasets, the ground-truth causal graphs are not available. For Adult and Nutrition datasets, we use the causal graphs built in previous works \cite{nabi2018fair, chiappa2019path, wang2021shapley}. For COMPAS dataset, we construct the causal graph with PC algorithm \cite{spirtes2000causation}, with directions on certain edges restricted and corrected according to domain knowledge. The PDAGs and rules we use to determine the causal directions are shown in the appendix. 

\subsection{Path Explanations for $\Delta_{DP}$}
\subsubsection{Baselines} Since there is no existing method specifically designed to explain disparity by causal paths. We adopt the following explanation methods as baselines: 
\begin{itemize}
	\item Feature-based Explanation by Shapley Values: we use different Shapley values (Independent Shapley Values(\textit{ISV}) in \cite{lundberg2017unified} and Asymmetric Shapley Values(\textit{ASV}) in \cite{frye2020asymmetric}) to calculate the feature-based contribution to disparity with Eq.(\ref{eq:dp_feat}). 
	% ISV assumes all features to be independent, while ASV consider the causal ordering of features and tend to assign contributions to the source nodes of the causal graph.
	\item Path-specific Explanations (PSE): we calculate the path-specific effect of each potential active path as the estimation of its contribution to disparity, following the calculation in \cite{chiappa2019path}. Since the calculation of path-specific effect requires that the causal relations among $X(\mathbf{P})$ must be identified in $\g$. We only report the quantitative results on synthetic, Adult and COMPAS datasets, where this condition is satisfied.
\end{itemize}

\subsubsection{Evaluation Metric} As path-specific effect (PSE) can also provide estimations of path contributions to disparity, we directly compare PSE and \methodname~with the following evaluation metrics.

\begin{enumerate}
	\item Accuracy: For synthetic dataset where the ground-truth path contributions is available, we evaluate the methods with the normalized root-mean-square error: $\frac{\sqrt{\sum_{p_i \in \mathbf{P}}(\theta_f(p_i) - \Phi_f(p_i))^2}}{\sqrt{\sum_{p_i \in \mathbf{P}}(\theta_f(p_i))^2}}$ where $\theta_f(p_i)$ is the ground-truth contribution of $p_i$.
	\item Efficiency: Efficiency is the property that the sum of the contribution values from all FACTs exactly equals to the total disparity (Property 1): $\sum_{p_i \in \mathbf{P}} \Phi_f(p_i) = \Delta_{DP}$. To show how close methods come to achieving efficiency, we compute normalized absolute difference between summation of path contributions and disparity: $\frac{|\sum_{p_i \in \mathbf{P}} \Phi_f(p_i) - \Delta_{DP}|}{|\Delta_{DP}|}$.
\end{enumerate}

\subsubsection{Quantitative Results on Estimation of Path Contributions} The result on accuracy is shown in Table \ref{tab:est_dp_ctb}. Our \methodname~outperforms PSE on different datasets and prediction models, especially when $f$ learns non-linear relation. The result on efficiency is shown in Table \ref{tab:eff_dp}. We can find that the summation of path explanations obtained by \methodname~is closer to the original value of disparity than baseline. 

\begin{figure}[ht]
	\centering
	\includegraphics[width=1.0\columnwidth]{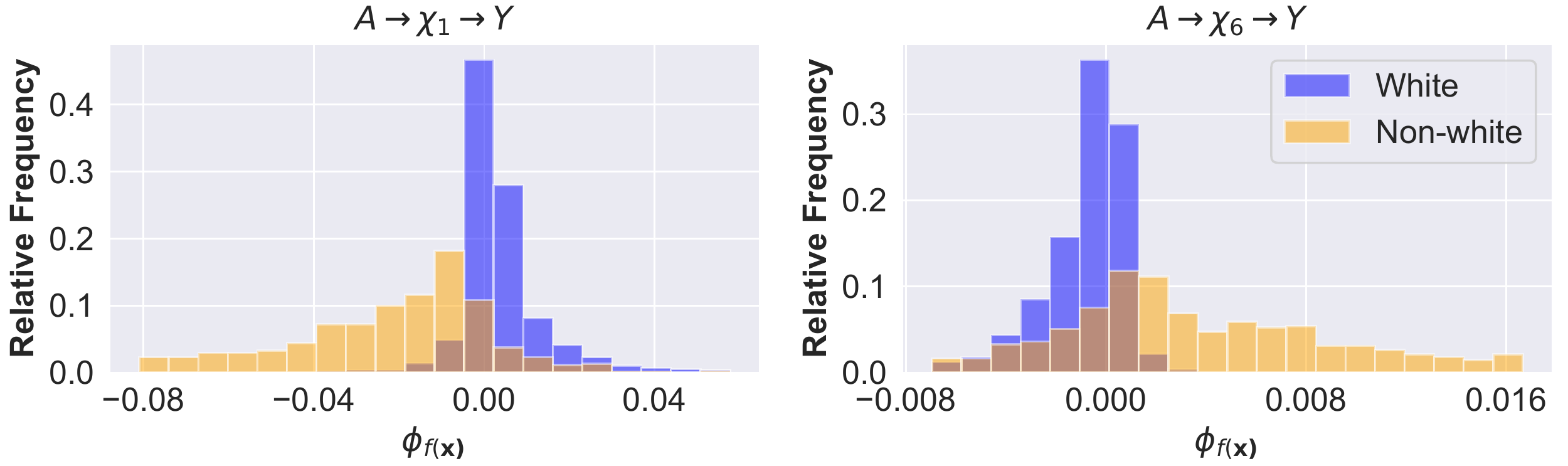}
	\caption{The histograms of $\phi_{f(\mathbf{x})}$ for two selected paths $A \rightarrow X_1 \rightarrow \hat{Y}$ and $A \rightarrow X_6 \rightarrow \hat{Y}$.}
	\label{fig:nutrition_hist}
	\vspace{-0.5em}
\end{figure}
% The y-axis is the relative frequency.

\subsubsection{Qualitative Analysis on Real Dataset}
We choose Nutrition as the case study in this section and the results on other datasets are shown in the appendix. We train a model $f$ without any fairness constraint on the dataset. Figure \ref{fig:nutrition_graph} illustrates the causal graph which contains the top paths contributing to $\Delta_{DP}$. Table \ref{tab:nutrition_feature_contribution} shows the feature contributions to $\Delta_{DP}$ obtained by ISV and ASV. The result of \methodname~is shown in Table \ref{tab:nutrition_path_contribution}. The first column of the table represents the path contributions to $\Delta_{DP}$, while the second column shows the path contributions to the utility. 

We can see that ${\chi}_1$,  which includes the features of economic status $\{$Poverty Idx, Food Program$\}$, has an important effect on $\Delta_{DP}$. Here Poverty Idx stands for the Poverty Index of a person and Food Program denotes whether he/she is qualified for food programs. Its contribution calculated from ASV is more dominant than that from ISV. This is because economic status, in addition to its direct impact, indirectly affects the output through sedimentation rate, as shown in Figure \ref{fig:nutrition_graph}. While ISV underestimates the impact of economic status and ASV mixes up all its impacts through different paths, \methodname~ can provide a more comprehensive and detailed analysis.

We also plot the histograms of $\phi_{f(\mathbf{x})}$ for selected paths in Figure \ref{fig:nutrition_hist}. We can see that $\phi_{f(\mathbf{x})}$ distributes differently across racial groups. 

\subsection{Fairness Learning through FACTs Selection}
\subsubsection{Baselines} We choose fair learning baselines as in \cite{begley2020explainability}.
\begin{itemize}
	\item UNFAIR: A baseline model without
	any fairness constraint.
	\item AFL \cite{zhang2018mitigating}: AFL is a fair learning algorithm which trains the model by maximizing the predictors ability to predict $Y$ while minimizing the adversary’s ability to predict $A$.
	\item RFL \cite{agarwal2018reductions}: Agarwal proposes an approach to fair classification, which optimize a constrained objective to achieve fairness.
\end{itemize}

For fair comparison, we use the same structure of MLP for all methods. For our method, we will first train an UNFAIR model $f$ and apply our algorithm on $f$ to get $\pathsubset^{*}$. Then we finetune $f$ and use $v_{f(\mathbf{x})}(\pathsubset^{*})$ as the new prediction result.

\begin{figure*}[ht]
	\centering
	\includegraphics[width=1.8\columnwidth]{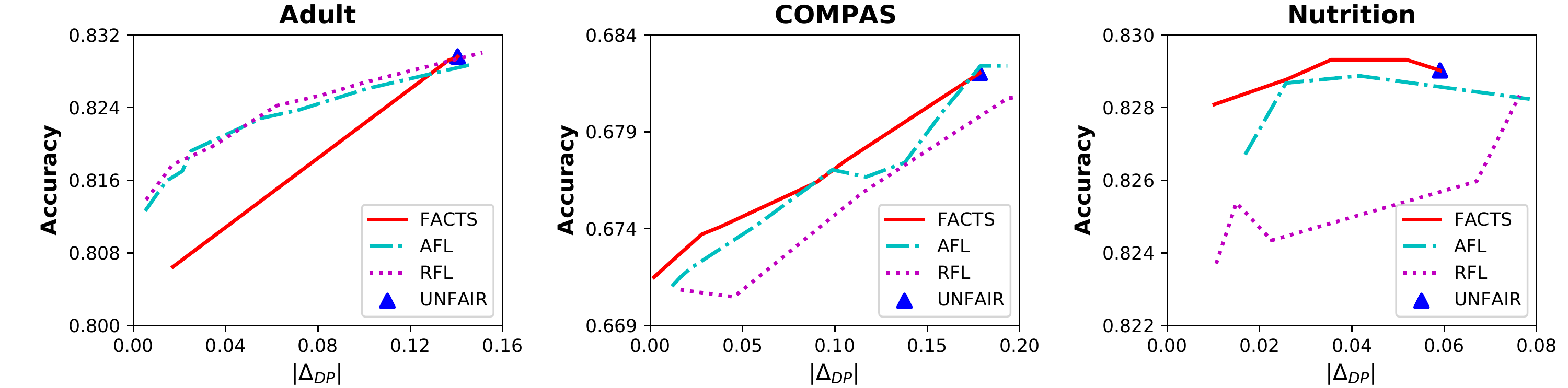}
	\caption{The accuracy-fairness trade-off curves for $\Delta_{DP}$ on various datasets. The upper-left corner (high accuracy, low disparity) is preferred. Curves are generated by searching on a range of fair coefficients (for example, $\lambda$ in Eq.(\ref{eq:loss}))}
	\label{fig:result_fl}
	\vspace{-0.5em}
\end{figure*}

\subsubsection{Evaluation} We follow the work of \cite{agarwal2018reductions} to evaluate the fair learning algorithms. For each method, we run experiments in a range of the weights of fairness constraints (such as $\lambda$ in Eq.(\ref{eq:loss})). Then we plot the accuracy-disparity curves for each method. 

\subsubsection{Results} 
We firstly compare the path contributions to utility and disparity from Table \ref{tab:nutrition_path_contribution} and use Figure \ref{fig:fl_path_selection} to illustrate how our algorithm works. $A \rightarrow {\chi}_1 \rightarrow \hat{Y}$ is the largest disparity contributor and makes a major contribution to utility. Other paths also contribute to $\Delta_{DP}$, but their contributions to utility are relatively low. Algorithm \ref{alg:fl_searching} can determine the optimal sets of selected paths under different $\lambda$ in Eq.(\ref{eq:loss}). In Figure \ref{fig:fl_path_selection}, we can see that when $\lambda = 0$, $\pathsubset^{*}$ contains all FACTs to maintain maximal utility. If $\lambda$ is large enough~($\lambda = 1.0$), all FACTs will be removed to meet this strict fairness constraint. As $\lambda = 0.1$, $\pathsubset^{*}$ only contains $A \rightarrow {\chi}_1 \rightarrow \hat{Y}$ and $A \rightarrow {\chi}_6 \rightarrow \hat{Y}$. The ratio $\frac{\Psi_{f}(p_i)}{\Phi_{f}(p_i)}$ of $A \rightarrow {\chi}_1 \rightarrow \hat{Y}$ is greater than $\lambda = 0.1$ so it is selected. $A \rightarrow {\chi}_6 \rightarrow \hat{Y}$ is also included because the direction of its $\Phi_{f}(p_i)$ is opposite to that of $A \rightarrow {\chi}_1 \rightarrow \hat{Y}$. It may be odd to keep two paths with different directions of $\Phi_{f}(p_i)$ to obtain a low absolute value of disparity. This reflects the limitation of treating disparity as a single metric of fairness.

We show the accuracy-fairness trade-off curves for demographic parity in Figure \ref{fig:result_fl}. Our fair learning algorithm with \methodname~ achieves comparable results as other fair learning algorithms. It outperforms baselines on COMPAS and Nutrition datasets. While on Adult dataset, most of $\Delta_{DP}$ comes from a single path (See appendix). After this path is removed, both $\Delta_{DP}$ and accuracy decrease obviously. Our algorithm can be completely explained by which paths are removed, while the baselines can only obtain black-box models. 

\section{Conclusion}
In this work, we propose to a novel framework to explain algorithmic fairness with the causal graph. In particular, we decompose the disparity into contributions from fairness-aware causal paths that link the sensitive attribute and model outcome on a causal graph. We propose an algorithm to identify those paths and calculate their contributions to the disparity. With the path explanations, we can gain more insight into the inequality of a machine learning algorithm and propose an interpretable method to achieve better trade-offs between utility and fairness.

%%
%% The acknowledgments section is defined using the "acks" environment
%% (and NOT an unnumbered section). This ensures the proper
%% identification of the section in the article metadata, and the
%% consistent spelling of the heading.
\begin{acks}
Weishen Pan, Sen Cui, and Changshui Zhang would like to acknowledge the funding by the National Key Research and Development Program of China (No. 2018AAA0100701) and Beijing Academy of Artificial Intelligence (BAAI). Fei Wang would like to acknowledge the support from Amazon Web Service (AWS) Machine Learning for Research Award and Google Faculty Research Award.
\end{acks}

%%
%% The next two lines define the bibliography style to be used, and
%% the bibliography file.
\bibliographystyle{ACM-Reference-Format}
\bibliography{main}

%%
%% If your work has an appendix, this is the place to put it.
\appendix

\section{Proofs}
\subsection{Proof of Proposition \ref{prop1}}
\begin{proof}
	Suppose there is no potential active path between two variables $X_i$ and $X_j$ on $\g$, but they are dependent. According to the faithful assumption, consider all directions of undirected edges on $\g$, there must be a DAG $\g^{'}$ that the conditional independence relationships among variables encoded in $\g^{'}$ are consistent with those inferred from observational data. That is, there are active paths from certain $X_i$ to $X_j$ on $\g^{'}$. Suppose $p^{'}$ is one of these active paths and $p$ is its corresponding path on $\g$. According to the definition, $p$ must be a potential active path, which violates the assumption.
\end{proof}

\subsection{Proofs of Proposition \ref{prop2} and \ref{prop3}}
To prove Proposition \ref{prop2} and \ref{prop3}, we first prove the following lemma:
\begin{lemma}
	\label{lemma1}
	In a DAG, if $p$ is an active path, then every subpath of $p$ is an active path.
\end{lemma}

\begin{proof}
	Suppose $p'$ is a subpath of $p$. For each $X$ that is a node on $p'$, then $X$ is also on $p$. Since $p$ is an active path, then $X$ is a non-collider. According to the generality of $X$, each node on $p'$ is a non-collider. So $p'$ is also an active path. 
\end{proof}

Then we will prove Proposition \ref{prop2} with Lemma \ref{lemma1}:
\begin{proof}{(Proposition \ref{prop2})} 
    Suppose $p$ is a potential active path and $p'$ is any subpath of $p$. We consider the three conditions in the definition of potential active path. If all edges on $p$ are directed and $p$ satisfies the definition of active path. Then all edges on $p'$ are also directed and $p'$ is active according to Lemma \ref{lemma1}, and thus a potential active path. Similarly, for the rest two cases, if $p$ satisfies the condition, we can get $p'$ also satisfies the same condition by Lemma \ref{lemma1}. So we can have $p'$ is a potential active path.
\end{proof}

\begin{proof}{(Proposition \ref{prop3})}
	Since $\g$ is a DAG, a potential active path is equivalent to an active path according to the definition. Assume $\exists X_i, X_j \in \mathbf{X}\left(\mathbf{P}\right), X_i \in \Adj(X_j)$, but neither of them is prior to the other with respect to $A$. According to the definition, at least one of the following conditions are satisfied:
	\begin{itemize}
		\item There is no active path from $A$ to $\hat{Y}$ through both $X_i$ and $X_j$. Without loss of generality, we let $X_i \rightarrow X_j$. Since $X_i \in \mathbf{X}\left(\mathbf{P}\right)$, there exists an active path from $A$ to $\hat{Y}$ through $X_i$. We consider its subpath which connects $A$ and $X_i$ and call it $p$. $p$ is an active path according to Lemma \ref{lemma1}. Considering a new path obtained by appending "$X_j \rightarrow \hat{Y}$" to $p$, it will be obviously an active path since $X_j$ and $\hat{Y}$ must be non-colliders. This new path is an active path from $A$ to $\hat{Y}$ and both $X_i$ and $X_j$ are on it. So the beginning condition does not hold.
		\item There exist at least two active paths from $A$ to $\hat{Y}$, one is as "$A \rightarrow \dots \rightarrow X_i \dots X_j \dots \rightarrow \hat{Y}$" and the other is as "$A \rightarrow \dots \rightarrow X_j \dots X_i \dots \rightarrow \hat{Y}$". The paths are pointed out from $A$ because $A$ is a source node in $\g$. Without loss of generality, we let $X_i \rightarrow X_j$. We denote the latter path as $p$. And we denote the subpath of $p$ from $X_j$ to $X_i$ as $p'$. $p'$ must be an active path according to Lemma \ref{lemma1}. There are three possibilities: $X_j \rightarrow \dots \rightarrow X_i$, if it is satisfied, $p'$ and the edge $X_i \rightarrow X_j$ will form a directed cycle, which violates that $\g$ is a DAG; $X_j \leftarrow \dots \leftarrow X_i$ or $X_j \leftarrow \dots \leftarrow X_k \rightarrow \dots \rightarrow X_i$, in both cases, $X_j$ will be a collider on $p$ and $p$ is not active. Since all the situation can not be satisfied, the beginning condition dose not hold.
	\end{itemize}
	
	In conclusion, neither of the conditions above can be satisfied. So the original assumption can not be satisfied, and $\mathbf{X}\left(\mathbf{P}\right)$ is completely ordered with respect to $A$ on $\g$.
\end{proof}

\subsection{Proofs of Properties}
\begin{proof}
	\textbf{Efficiency}: from Eq. (\ref{eq:calculation}), we will have $\sum_{p_i \in \mathbf{P}} \phi_{f(\mathbf{x})}(p_i) = v_{f(\mathbf{x})}(\mathbf{P}) -  v_{f(\mathbf{x})}(\emptyset)$. Since $v_{f(\mathbf{x})}(\mathbf{P}) = f(\mathbf{x})$, we have:
	\begin{equation}
		\begin{aligned}
			\sum_{p_i \in \mathbf{P}} \Phi_{f}(p_i) &=  (\mathbb{E}_{\mathbf{X}|A= 1}[v_{f(\mathbf{x})}(\mathbf{P})] - \mathbb{E}_{\mathbf{X}|A=0}[v_{f(\mathbf{x})}(\mathbf{P})]) \\
			& - (\mathbb{E}_{\mathbf{X}|A= 1}[v_{f(\mathbf{x})}(\emptyset)] - \mathbb{E}_{\mathbf{X}|A=0}[v_{f(\mathbf{x})}(\emptyset)])  
		\end{aligned}
	\end{equation}
	\noindent where $\mathbb{E}_{\mathbf{X}|A= 1}[v_{f(\mathbf{x})}(\mathbf{P})] - \mathbb{E}_{\mathbf{X}|A=0}[v_{f(\mathbf{x})}(\mathbf{P})] = \Delta_{DP}$.
	And when we calculate $v_{f(\mathbf{x})}(\emptyset)$, we only use $E_i$ and $\bar{\mathbf{X}}(\mathbf{P}_{Y}))$, which are independent to $A$. So $v_{f(\mathbf{x})}(\emptyset)$ is also independent to $A$, $\mathbb{E}_{\mathbf{X}|A= 1}[v_{f(\mathbf{x})}(\emptyset)] - \mathbb{E}_{\mathbf{X}|A=0}[v_{f(\mathbf{x})}(\emptyset)] = 0$.
	
	\noindent \textbf{Linearity}: In the calculation of $v_{(\alpha f_1 + \beta f_2)(\mathbf{x})}(\pathsubset)$, $(\alpha f_1 + \beta f_2)(\mathbf{x})$ can be written as $\alpha f_1(\mathbf{x}) + \beta f_2(\mathbf{x})$. Then we have $v_{(\alpha f_1 + \beta f_2)(\mathbf{x})}(\pathsubset) = \alpha v_{f_1(\mathbf{x})}(\pathsubset) + \beta v_{f_2(\mathbf{x})}(\pathsubset)$ for all possible $\pathsubset$. With Eq.(\ref{eq:calculation}) and (\ref{eq:dp_fact}), we can prove the Linearity.
	
	\noindent \textbf{Nullity}: If $v_{f(\mathbf{x})}(\pathsubset \cup \{p_i\}) = v_{f(\mathbf{x})}(\pathsubset),~\forall \mathbf{x}, \pathsubset \subset \mathbf{P} \backslash p_i$, obviously all the item in the summation of Eq. (\ref{eq:calculation}) will be 0. Then $\phi_{f(\mathbf{x})}(p_i)$ and $\Phi_{f}(p_i)$ are both 0.
\end{proof}

\section{Extension to Other Disparity}
As discussed in previous work by Baer {\em et al.} \cite{baer2019fairness}, equalized odds and equalized opportunity satisfy when $\hat{Y} \ind A | Y$. In fact,  $\hat{Y} \ind A | Y$ can also induce accuracy parity. To study the conditional dependence between $\mathbf{X}$ and $\hat{Y}$ on $Y$, we first introduce a new causality concept:

\noindent \textbf{Spouse.} If both $X_i$ and $X_j$ are parents of $X_k$, then $X_i$ is a \textit{} of $X_j$ and $X_j$ is a \textit{spouse} of $X_i$ with child $X_k$. 

\noindent \textbf{Potential Active Paths Relative to $Y$.} The definition of potential active path relative to $Y$ can be obtained by replace the notions "active path" with "active path relative to $Y$" in the definition of potential active path. To save space, we only write the first condition as an example: if a path $p$ is a directed path and satisfies the definition of active path relative to $Y$ in Section\ref{sec:causal_models}, then it is a potential active path relative to $Y$.

The potential active path and potential active path relative to $Y$ have the following relation:
\begin{proposition}
	\label{propa1}
	For any path on $\g$ which does not pass through $Y$, if it is a potential active path, then it is a potential active path relative to $Y$.
\end{proposition}

\begin{proposition}
	\label{propa2}
	If there is no potential active path relative to $Y$ between two variables (neither of them is $Y$) on $\g$, then the two variables are conditional independent on $Y$.
\end{proposition}

We denote the set of potential active paths from $A$ to $\hat{Y}$ relative to $Y$ as $\mathbf{P}_{Y}$. We use an example in Figure \ref{fig:example_pos_active_cond} to further illustrate the definition. We can get: $\mathbf{P}_{Y} = \{A \rightarrow X_1 \rightarrow \hat{Y}, A \rightarrow X_1 \rightarrow Y \leftarrow X_2 \rightarrow \hat{Y}\}$. $A \rightarrow X_1 \rightarrow \hat{Y}$ obviously satisfies the definition, while $A \rightarrow X_1 \rightarrow Y \leftarrow X_2 \rightarrow \hat{Y}$ is not so intuitive. It is not an active path from $A$ to $\hat{Y}$. But if we consider the conditioning set $\{Y\}$, we can see $Y$ is a collider on this path and belongs to the conditional set. While other nodes $\{X_1,X_2\}$ are non-colliders and not in $\{Y\}$, according to the definition, $A \rightarrow X_1 \rightarrow Y \leftarrow X_2 \rightarrow \hat{Y}$ is an active path relative to $Y$.

We can propose a similar definition of ordering with respect to $A$ conditioned on $Y$. $X_i \succ_{A|Y} X_j$ means $X_i$ is prior to $X_j$ with respect to $A$ conditioned on $Y$. When we further consider the information of $A$ passes to $\mathbf{P}_{Y}$. It passes through spouses with child $Y$ ($X_1 \rightarrow Y \leftarrow X_2$) as well as from direct predecessor to direct successor ($A \rightarrow X_1$). So we propose new concepts other than direct predecessor/successor.

If the two following conditions satisfy, then we call $X_i$ is an {\em informative successor} of $X_j$ with respect to $A$ relative to $Y$: 1) $X_j \succ_{A|Y} X_i$; 2) $X_j$ is adjacent to $X_i$ or $X_i, X_j$ are spouses with child $Y$. And $X_j$ is called an {\em informative predecessor} of $X_i$ with respect to $A$ relative to $Y$. The set of all informative successors of $X_j$ is denoted as $\Is_{A|Y}(X_i)$. The set of all informative predecessors of $X_j$ is denoted as $\Ip_{A|Y}(X_i)$. For Figure \ref{fig:example_pos_active_cond}, we have $\Ip_{A|Y}(X_1) = \{A\}, \Ip_{A|Y}(X_2) = \{A, X_1\}$. We will write $\Ip_{A|Y}(X_i)$ as ${\Ip}_i$ and $\Is_{A|Y}(X_i)$ as ${\Is}_i$ for simplicity in the following.

\begin{figure}
	\centering
	\subfigure[]{
		\includegraphics[width=0.22\columnwidth]{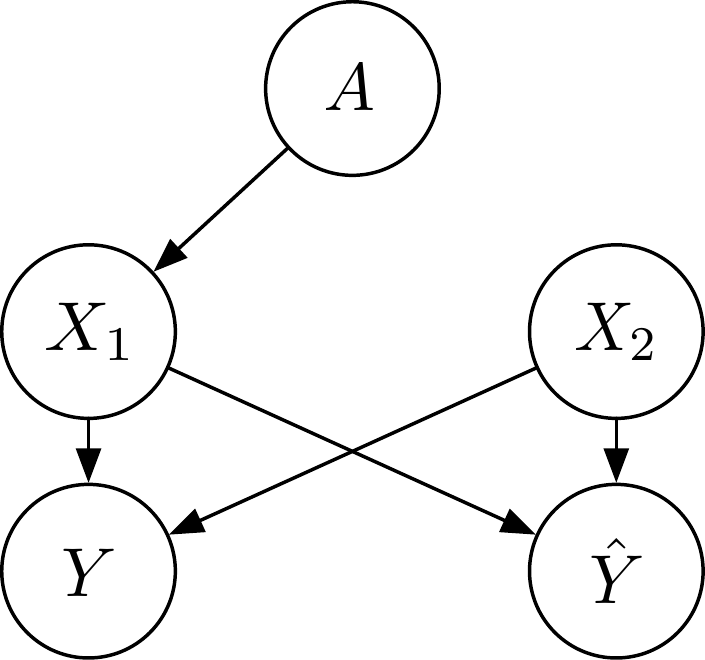}
		\label{fig:example_pos_active_cond}
	}
	\quad \quad
	\subfigure[]{
		\includegraphics[width=0.2\columnwidth]{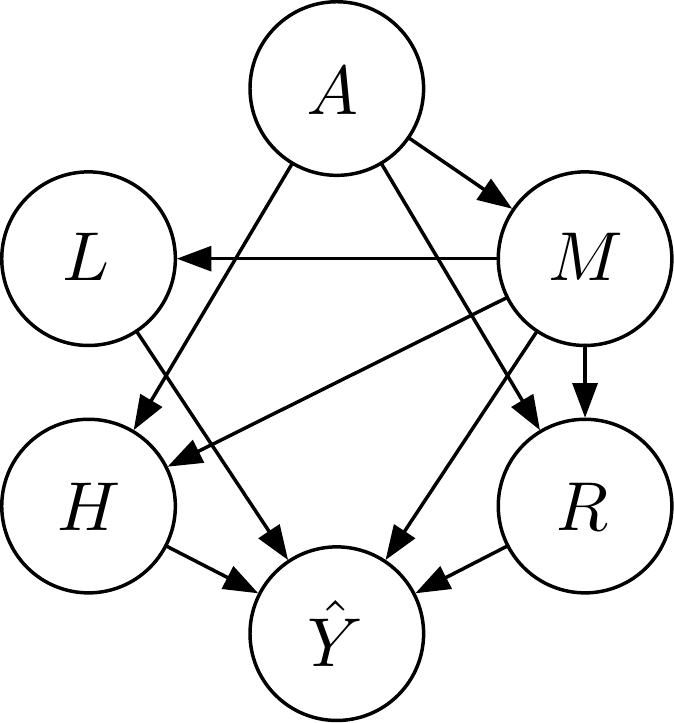}
		\label{fig:adult_graph}
	}
	\quad \quad
	\subfigure[]{
		\includegraphics[width=0.2\columnwidth]{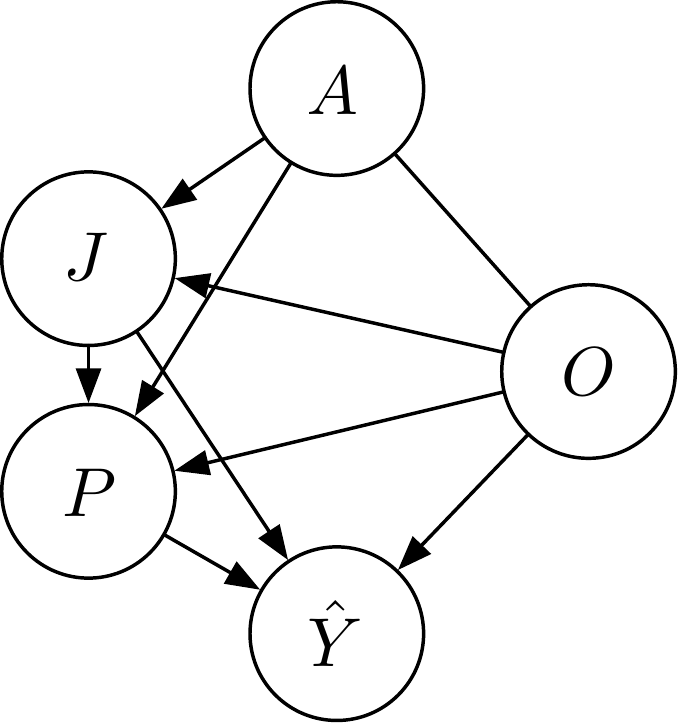}
		\label{fig:compas_graph}
	}
	
	\caption{(a): A PDAG to illustrate the concepts when considering $\{Y\}$ as the condition set. (b) PDAG of $\mathbf{P}$ on Adult dataset. A:sex, M:marital status, L:level of education, H:working hours per week, R:relationship; (c) PDAG of $\mathbf{P}$ on COMPAS dataset. A:race, O:age, J: triple of numbers of juvenile felonies/juvenile misdemeanors/other juvenile convictions, P:prior crimes}
	\label{fig:other_graph}
\end{figure}

\noindent \textbf{Completely Ordered With Respect To $A$ Relative to $Y$.} $\mathbf{X}\left(\mathbf{P}_{Y}\right)$ is defined to be completely ordered with respect to $A$ relative to $Y$ on $\g$ if: $\forall X_i, X_j \in \mathbf{X}\left(\mathbf{P}_{Y}\right)$, we have if $X_i$ and $X_j$ are adjacent or spouses with child $Y$, one of them must be the informative successor or predecessor of the other.

After defining these concepts, we can directly replace the concept of $\Dp_{i}$ with $\Ip_{i}$ during the step of grouping features. When we try to decompose $f(\mathbf{x})$ into the path contributions conditioned on $Y$. We need to learn a two sets of functions $\{g_{X_i|Y=1}\}$ and $\{g_{X_i|Y=0}\}$, corresponding to different values of $Y$. And we need to infer $X_i$ as $X_i = g_{X_i|Y}({\Ip}_{i}, \bar{\mathbf{X}}(\mathbf{P}_{Y}))$. Thus we have $\phi_{f(\mathbf{x}),y}(p_i)$ obtain by replacing the function $g_{X_i}$ with $g_{X_i|Y}$ in the calculation of $\phi_{f(\mathbf{x}),y}(p_i)$.

For equalized opportunity, we compute the path contribution to $\Delta_{OP}$ as:
\begin{equation}
	\label{eq:eo_fact}
	\Phi_{f}(p_i) = \mathbb{E}_{\mathbf{X}|Y=1, A=1}[\phi_{f(\mathbf{x}),y}(p_i)] - \mathbb{E}_{\mathbf{X}|Y=1,A=0}[\phi_{f(\mathbf{x}),y}(p_i)]
\end{equation}

For accuracy parity, we have:
\begin{equation}
	\label{eq:ac_fact}
	\Phi_{f}(p_i) = \mathbb{E}_{\mathbf{X}| A=1}[\phi_{f_{y}(\mathbf{x}),y}(p_i)] - \mathbb{E}_{\mathbf{X}|A=0}[\phi_{f_y(\mathbf{x}),y}(p_i)]
\end{equation}

For equalized odds, we need to consider the disparity conditioned on $Y = 1$ and $Y = 0$ respectively.

\section{Implementation Details}
\noindent \textbf{Generation of Synthetic Data.} In experiments on effiency, we generate the data as follows: a feature $X_i$ is randomly connected to $X_j$ (with $X_j$ pointing to $X_i$) with $0.2$ probability if $i > j$, otherwise $0$. A feature $X_i$ is randomly connected to $A$ (with $A$ pointing to $X_i$) with $0.4$ probability. The generative function of each feature is linear. In experiments on accuracy, we constrain symmetric structure and parameters to obtain the ground-truths. In $S_1$, the probability of $Y = 1$ is calculated by the a linear function of features (normalized to be the in range of $[0,1]$). While in $S_2$, the probability is calculated with additional non-linear transformation.

\noindent \textbf{Model Parameters.} When we implement $f$ as MLP, we use the network with one hidden layer and the dimensions of hidden layer are 32/16/8/16 for Synthetic/Adult/COMPAS/Nutrition respectively. The models are implemented and trained with the Python package {\em scikit-learn}. We choose ADAM to optimize the model. 

\noindent \textbf{Causal Discovery Results on Real-world datasets.} Since there is no widely-used causal graph for COMPAS dataset, we generate a PDAG by the following procedure: first we run the PC algorithm on the data, then we determine and correct the direction following the rules: $\{$race, age, gender$\} \rightarrow \{$numbers of prior crimes, numbers of juvenile felonies/juvenile misdemeanors/other juvenile conviction$\} \rightarrow \{$prior crimes$\} \rightarrow \{$original charge degree$\}$.

\noindent \textbf{Estimation of $\Phi_{f}$.} Since \methodname~and all baselines are expensive to compute exactly, we use a Monte Carlo approximation of Eq.( \ref{eq:calculation}). In particular, we conduct breadth first search and sample 100 orderings from $\Pi$ and average across those orderings. For mixed-type data, we binarize all categorical features and follow the recent work of Wang {\em et al.} \cite{wang2021shapley} to formulate $g_{X_i}$ and $E_i$.

\section{Additional Experiment Results}

\begin{figure}
	\centering
	\subfigure[]{
	\setlength{\tabcolsep}{0.01\columnwidth}{
	    \scriptsize
		\begin{tabular}{ccc}
			\toprule
			{Paths} & $\Phi_{f}(p_i)$ & $\Psi_{f}(p_i)$ \\
			\midrule
			$A \rightarrow M \rightarrow \hat{Y}$ & \cellcolor{red!100} 0.116 & 0.034\\
			$A \rightarrow H \rightarrow \hat{Y}$ & \cellcolor{red!20} 0.024 & 0.005\\
			$A \rightarrow M \rightarrow L \rightarrow \hat{Y}$ & \cellcolor{red!10} 0.011 & 0.001\\
			$A \rightarrow M \rightarrow H \rightarrow \hat{Y}$ & \cellcolor{red!10} 0.010 & 0.001\\
			$A \rightarrow R \rightarrow \hat{Y}$ & \cellcolor{blue!10} -0.013 & <0.001 \\
			$A \rightarrow M \rightarrow R \rightarrow \hat{Y}$ & \cellcolor{blue!25} -0.031 & <0.001 \\
			\bottomrule
		\end{tabular}
		\label{tab:adult_table}
	}}
	\subfigure[]{
		\setlength{\tabcolsep}{0.01\columnwidth}{
		    \scriptsize
			\begin{tabular}{ccc}
				\toprule
				{Paths} & $\Phi_{f}(p_i)$ & $\Psi_{f}(p_i)$ \\
				\midrule
				$A \rightarrow O \rightarrow \hat{Y}$ & \cellcolor{red!100} 0.109 & 0.010\\
				$A \rightarrow P \rightarrow \hat{Y}$ & \cellcolor{red!60} 0.066 & 0.005\\
				$A \rightarrow J \rightarrow P \rightarrow \hat{Y}$ & \cellcolor{red!15} 0.017 & 0.004\\
				$A \rightarrow O \rightarrow J \rightarrow P \rightarrow \hat{Y}$ & \cellcolor{red!5} 0.005 & 0.001\\
				$A \rightarrow O \rightarrow J \rightarrow \hat{Y}$ & \cellcolor{red!5} 0.005 & <0.001 \\
				$A \rightarrow O \rightarrow P \rightarrow \hat{Y}$ & \cellcolor{blue!15} -0.017 & <0.001 \\
				\bottomrule
			\end{tabular}
			\label{tab:compas_table}
		}}
		
	\caption{Path explanations on: (a)~Adult and (b) COMPAS dataset. The first column shows the contributions of paths to $\Delta_{DP}$($\Phi_{f}(p_i)$), the second column shows the contributions of paths to utility($\Psi_{f}(p_i)$).}
	\label{fig:other_table}
	\vspace{-0.5em}
\end{figure}

The results on Adult and COMPAS dataset are shown in Figure \ref{fig:adult_graph} and \ref{fig:compas_graph}. Due to limited space, we only display the top paths of $\mathbf{P}$ on the graphs. And the corresponding path contributions are shown in Figure \ref{fig:other_table}. The edge $A-O$ in \ref{fig:compas_graph} is undirected. In Adult dataset, marital status acts as the essential feature contributing to disparity. While in COMPAS dataset, age is the largest contributor.

\end{document}